%% file: main.tex
\documentclass[twoside,11pt]{article}

\usepackage{linjmlr2e}

\usepackage[utf8]{inputenc} 
\usepackage[T1]{fontenc}    
\usepackage{hyperref}       
\usepackage{url}            
\usepackage{booktabs}       
\usepackage{amsfonts}       
\usepackage{nicefrac}       
\usepackage{microtype}      

\usepackage{graphicx}
\usepackage{subcaption}

\usepackage{amsmath}
\usepackage{amssymb}
\usepackage[capitalize]{cleveref}
\usepackage{breqn}
\usepackage{xspace}
\usepackage{multirow}
\usepackage[pdftex,dvipsnames,table]{xcolor}
\usepackage[colorinlistoftodos,prependcaption,textsize=small]{todonotes}
\usepackage{graphicx}
\usepackage{footnote}
\makesavenoteenv{tabular}
\makesavenoteenv{table}
\usepackage{xargs} 
\usepackage{threeparttable}
\usepackage{tikz}
\usetikzlibrary{calc, matrix, shapes, arrows, positioning, chains, decorations, 
	shapes.gates.logic.US}
\usepackage{mdframed}
\usepackage{algorithm}
\usepackage{algorithmic}
\usepackage{adjustbox}

\input{mysymbol.sty}

\renewcommand{\algorithmicrequire}{\textbf{Input:}}
\renewcommand{\algorithmicensure}{\textbf{Output:}}

\newtheorem{assump}{Assumption}

\crefname{assump}{Assumption}{Assumptions}

\newcommand{\expect}{\mathbb{E}}
\newcommand{\constraint}{\mathcal{K}}
\renewcommand{\var}{\operatorname{Var}}
\renewcommand{\sign}{\operatorname{sgn}}

\newcommand{\AlgQFW}{\texttt{Quantized Frank-Wolfe}\xspace}
\newcommand{\Alg}{\texttt{Stochastic Quantized Frank-Wolfe}\xspace}
\newcommand{\fAlg}{\texttt{Finite-Sum Quantized Frank-Wolfe}\xspace}
\newcommand{\qfw}{QFW\xspace}
\newcommand{\sqfw}{S-QFW\xspace}

\newcommand{\sscheme}{\texttt{Sign Encoding Scheme}\xspace}
\newcommand{\mscheme}{\texttt{Partition Encoding Scheme}\xspace}

\newcommand{\signfw}{\texttt{SignFW}\xspace}
\newcommand{\qfwshort}{\texttt{QFW}\xspace}

\begin{document}


\title{Quantized Frank-Wolfe: Faster Optimization, Lower Communication, and 
Projection Free}

\author{\name Mingrui Zhang \email mingrui.zhang@yale.edu \\
       \addr Department of Statistics and Data Science\\
       Yale University\\
       New Haven, CT 06511
       \AND
       \name Lin Chen \email lin.chen@yale.edu \\
       \addr Yale Institute for Network Science \\
       \addr Department of Electrical Engineering\\
       Yale University\\
       New Haven, CT 06511
   	   \AND
       \name Aryan Mokhtari \email  aryanm@mit.edu\\
       \addr Laboratory for Information and Decision Systems \\
       Massachusetts Institute of Technology \\
       Cambridge, MA 02139
  	   \AND
  	   \name Hamed Hassani \email hassani@seas.upenn.edu \\
  	   \addr Department of Electrical and Systems Engineering \\
  	   University of Pennsylvania \\
  	   Philadelphia, PA 19104 
  	   \AND
  	   \name Amin Karbasi \email amin.karbasi@yale.edu \\
  	   \addr Department of Electrical Engineering and Computer Science \\
  	   Yale University\\
  	   New Haven, CT 06511
}


\maketitle
\begin{abstract}
	How can we efficiently mitigate the overhead of gradient communications in 
	distributed 
	optimization? This problem is at the heart of training scalable machine 
	learning models and has 
	been 
	mainly studied  in the unconstrained setting. In this paper, we propose 
	\AlgQFW (QFW), the first projection-free and 
	communication-efficient algorithm for solving \textit{constrained} 
	optimization problems at 
	scale. We consider 
	both convex and non-convex objective functions, expressed as a finite-sum 
	or more generally a 
	stochastic optimization problem, and provide strong theoretical guarantees 
	on the convergence 
	rate of 
	QFW. This 
	is accomplished by
	proposing novel  quantization schemes that efficiently compress gradients while 
	controlling  
	the 
	noise variance introduced during this process. Finally, we empirically validate 
	the 
	efficiency of 
	QFW in terms of communication and the quality of returned solution against 
	natural 
	baselines. 
	
\end{abstract}
\section{Introduction}
The Frank-Wolfe (FW) method \citep{frank1956algorithm}, also known as 
conditional gradient, has recently received considerable 
attention in the machine learning community, as a projection free algorithm for various \textit{constrained}  convex 
\citep{jaggi2013revisiting,garber2014faster,lacoste2015global,garber2015faster, 
	hazan2016variance,mokhtari2018stochastic} and non-convex 
\citep{lacoste2016convergence,reddi2016stochastic} optimization problems.
In order to apply the FW algorithm to large-scale problems (\emph{e.g.}, training deep neural 
networks\citep{ravi2018constrained,schramowski2018neural,berrada2018deep}, 
RBMs\citep{ping2016learning}) parallelization is unavoidable. To this end, 
distributed FW variants have been 
proposed for specific problems, \emph{e.g.}, online 
learning \citep{zhang2017projection}, learning low-rank matrices 
\citep{zheng2018distributed}, and optimization under block-separable constraint 
sets \citep{wang2016parallel}. 
A significant performance bottleneck of distributed optimization methods is the 
cost of communicating gradients, typically handled by using a parameter-server 
framework.  Intuitively, if each worker in the distributed 
system transmits the entire gradient, then at least $d$ floating-point numbers 
are communicated for each worker, where $d$ is the dimension of the problem. This 
communication 
cost can be a huge burden on the performance of parallel optimization 
algorithms 
\citep{chilimbi2014project,seide20141,strom2015scalable}. To circumvent this 
drawback, communication-efficient parallel algorithms have received 
significant attention. One major approach is to quantize the 
gradients while maintaining sufficient information 
\citep{de2015taming,abadi2016tensorflow,wen2017terngrad}. For 
\textit{unconstrained} optimization, when projection is not required for implementing Stochastic Gradient Descent (SGD), several communication-efficient distributed methods have been proposed, including 
QSGD \citep{alistarh2017qsgd}, SIGN-SGD 
\citep{bernstein2018signsgd}, and Sparsified-SGD~\citep{stich2018sparsified}. 

In the constrained setting, and in particular for distributed FW algorithms, 
the communication-efficient versions were only studied for 
specific problems such as sparse learning 
\citep{bellet2015distributed,lafond2016d}. 
In this paper, however, we develop \AlgQFW (QFW), a general 
communication-efficient distributed FW for both convex and non-convex objective 
functions. We study the performance of QFW in in two widely 
recognized 
settings: 1) stochastic and 2) finite-sum optimization.
Let $\ccalK\subseteq \reals^d$ be the constraint set.
For \textit{constrained stochastic optimization} the goal is to 
solve
\begin{equation}\label{stochastic_problem}
\min_{x \in \constraint} f(x)\ :=\ \min_{x \in \constraint} \expect_{z\sim 
P}[\tilde{f}(x,z)],
\end{equation}
where $x\in \reals^d$ is the optimization variable,  $Z \in 
\reals^q$ is a random variable drawn from a  distribution $P$, 
which determines the choice of a stochastic function 
$\tilde{f}:\reals^d\times \reals^q\to \reals$. For \textit{constrained 
finite-sum optimization} we further assume that  $P$  is a uniform distribution 
over $[N] = \{1, 2, \cdots, N\}$ and the goal is to solve a special case of 
Problem~\eqref{stochastic_problem}, namely, 
\begin{equation}\label{finite_sum_problem}
\min_{x \in \constraint} f(x)\ :=\ \min_{x \in \constraint} \frac{1}{N} 
\sum_{i=1}^N f_i(x).
\end{equation} 
In parallel settings, we suppose that there is a master and $M$ 
workers, and each worker maintains a local copy of $x$.
At every iteration of the stochastic case, each 
worker 
has access to independent stochastic 
gradients of $f$; whereas in the finite-sum case, we assume $N=Mn$, thus the 
objective function can be 
decomposed as $f(x)=\frac{1}{Mn}{\sum_{m\in[M],i\in[n]}f_{m,i}(x)}$, 
and each worker $m$ has access to the exact gradients of $n$ component 
functions $f_{m,i}(x)$ for all 
$i \in [n]$. 
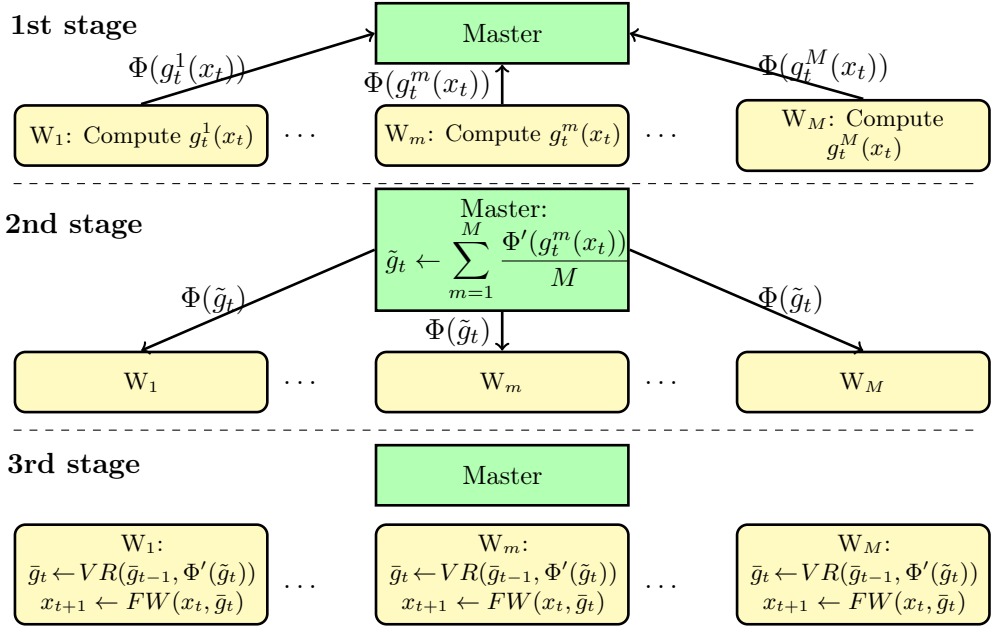
\begin{figure}[t]
	\begin{center}
		\input{challenges_2_app.tex} 
	\end{center}
	\caption{Stages of our general Quantized Frank-Wolfe scheme at time $t$. 
		In the first stage, each worker $m$ computes its local gradient 
		estimation 
		$g_t^m(x_t)$ and sends the quantized version $\Phi(g_t^m(x_t))$ to the 
		master node. In the second stage, master computes the average of 
		decoded 
		received signals $\Phi^{\prime}(g_t^m(x_t))$, \emph{i.e.}, $\tilde{g}_t 
		\gets 
		\frac{1}{M}\sum_{m=1}^M {\Phi^{\prime}(g_t^m(x_t))}$ and then sends its 
		quantized version $\Phi(\tilde{g}_t)$ to the workers. In the third 
		stage, 
		workers use the decoded gradient average computed by all workers 
		$\Phi'(\tilde{g}_t)$ and their previous gradient estimation 
		$\bar{g}_{t-1}$ to update their new gradient estimation $\bar{g}_{t}$ 
		via a variance reduction (VR) scheme. Once the variance reduced 
		gradient 
		approximation $\bar{g}_{t}$ is evaluated, workers compute the new 
		variable 
		$x_{t+1}$ by following the update of Frank-Wolfe (FW).}\label{fig:workflow2}
\end{figure}

This way the task of computing gradients is divided among the 
workers.  The 
master node aggregates local gradients from the workers, and sends the 
aggregated gradients back to them so that each worker can update the 
model (\emph{i.e.}, their own iterate) locally. 
Thus, by 
transmitting quantized gradients, we can reduce the communication complexity 
(\emph{i.e.}, number of transmitted bits) 
significantly. The workflow diagram of the distributed quantization scheme is 
summarized in Figure \ref{fig:workflow2}. 
Finally, we should highlight that there is a 
trade-off between gradient 
quantization and information flow. Intuitively, more intensive 
quantization reduces the communication cost, but also loses more 
information, which may decelerate the convergence rate. 

\textbf{Our contributions:} In this paper, we propose a novel 
distributed projection-free framework that handles quantization for constrained convex and 
non-convex optimization problems in stochastic and finite-sum cases. It is 
well-known that 
unlike projected gradient-based methods, FW methods 
may diverge when fed with stochastic gradient \citep{mokhtari2018stochastic}. 
Indeed, a similar issue arises in a distributed setting where nodes exchange 
\textit{quantized gradients} which are noisy estimates of the gradients. By 
incorporating appropriate variance reduction techniques in 
different settings, we show that with quantized gradients, we can obtain a 
provably convergent method which  preserves the 
convergence rates of the vanilla unquantized method in most cases. We believe 
our work presents the first quantized, distributed, and projection-free method, 
in contrast to all the previous works which consider quantization in the 
unconstrained setting. 
Our theoretical results for \AlgQFW (QFW) are summarized in   
\cref{tab:contribution}, where the SFO complexity is the 
required number of stochastic gradients in stochastic case, and the IFO 
complexity is the number of exact gradients for component functions in 
finite-sum case. To be more specific, we show that (i) QFW improves the IFO complexity 
$\mathcal{O}(1/\epsilon^2)$  of the SVRF method 
\citep{hazan2016variance} to 
$\mathcal{O}(n/\epsilon)$ for finite-sum convex case, by using the newly 
proposed SPIDER variance reduction 
technique; (ii) QFW preserves the SFO/IFO complexities of the SFW algorithm 
\citep{mokhtari2018stochastic} for stochastic convex case, and the accelerated 
NFWU method \citep{shen2019complexities} for finite-sum non-convex case; (iii) QFW has slightly worse SFO complexity $\mathcal{O}(1/\epsilon^4)$ than 
that of SVFW-S \citep{reddi2016stochastic}, $\mathcal{O}(1/\epsilon^{10/3})$, 
for the stochastic non-convex case, while it uses quantized gradients. 


\begin{table*}[t]
	\centering
	\caption{SFO/IFO Complexity and average communication bits in different 
		settings, where $M$ is the number of workers, 
		$z_{1}=\lceil\log_2[({\sqrt{n}dT^2}/{M})^{1/2}+1]\rceil, 
		z_{2}=\lceil\log_2[(\sqrt{n}dT^2)^{1/2}+1]\rceil, z_{3}=\lceil 
		\log_2[({4\sqrt{n}dT}/{M})^{1/2}+1] \rceil, z_{4}=\lceil 
		\log_2[(4\sqrt{n}dT)^{1/2}+1] \rceil$.}
	\label{tab:contribution}
		\begin{tabular}{|c|c|c|c|}   	
			\hline
			Setting & Function   & SFO/IFO Complexity  & Average Bits  \\ 
			\hline 
			stoch. & convex  & $\mathcal{O}(1/\epsilon^{3})$ & 
			$(M + 1)(2d+ 32)$ \\ 
			\hline
			stoch. & non-convex  & 
			$\mathcal{O}(1/\epsilon^{4})$ 
			& $(M + 1)(2d + 32)$ \\ 
			\hline
			finite-sum & convex  & $\mathcal{O}(n/\epsilon)$ & 
			$d(Mz_{1}+z_{2})+(M+1)(d+32)$ \\
			\hline
			finite-sum & non-convex & 
			$\mathcal{O}(\sqrt{n}/\epsilon^{2})$ & 
			$d(Mz_{3}+z_{4})+(M+1)(d+32)$ \\
			\hline 
		\end{tabular}
\end{table*}

\input{quantization.tex}
\input{stochastic.tex}

\input{finite_sum.tex}

\input{experiment.tex}
\section{Conclusion}
In this paper, we developed \AlgQFW (\qfw), the first general-purpose 
projection-free 
and communication-efficient 
framework for constrained optimization. Along with proposing various 
quantization schemes, \qfw can address both convex and non-convex 
optimization settings in  stochastic and finite-sum cases. We provided 
theoretical guarantees on the convergence rate of \qfw and validated its 
efficiency  empirically 
on training multinomial logistic regression and neural networks. Our 
theoretical results highlighted the importance of variance reduction techniques 
to stabalize Frank Wolfe and achieve a sweet trade-off between the 
communication complexity and convergence rate in distributed settings.


\newpage
\bibliographystyle{plainnat}
\bibliography{main}

\newpage
\onecolumn
\appendix
%
%
%


\section{Proof of \cref{lem:var_multi}}\label{app:lem_var_multi}
\begin{proof}
	For any given vector $g\in \reals^d$, the ratio ${|g_i|}/{\|g\|_\infty}$ 
	lies in an 
	interval of the form $[l_i/s,{(l_i+1)}/{s}]$ where $l_i\in \{0,1,\dots, 
	s-1\}$. Hence, for that specific $l_i$, the following inequalities 
	\begin{align}\label{helpppp_1}
	\frac{l_i}{s} \leq \frac{|g_i|}{\|g\|_\infty} \leq \frac{l_i+1}{s}
	\end{align}
	are satisfied. Moreover, based on the probability distribution of $b_i$ we 
	know that  
	\begin{align}\label{helpppp_2}
	\frac{l_i}{s} \leq b_i \leq \frac{l_i+1}{s}.
	\end{align}
	Therefore, based on the inequalities in \eqref{helpppp_1} and 
	\eqref{helpppp_2} we can write
	\begin{align}\label{helpppp_3}
	-\frac{1}{s} \leq \frac{|g_i|}{\|g\|_\infty}- b_i \leq \frac{1}{s}
	\end{align}
	Hence, we can show that the variance of $s$-\mscheme is upper bounded by
	\begin{align}
	\var[\phi^{\prime}(g)|g]
	&= \expect[\|\phi^{\prime}(g)-g\|^2|g] \nonumber\\
	&= \sum_{i=1}^d \expect[(g_i - \sign(g_i)b_i\|g\|_\infty)^2|g] \nonumber\\
	&= \sum_{i=1}^d \expect[(|g_i| - b_i\|g\|_\infty)^2|g] \nonumber\\
	&= \sum_{i=1}^d \|g\|_\infty^2 
	\expect\left[\left(\frac{|g_i|}{\|g\|_\infty} - b_i\right)^2\mid g\right] 
	\nonumber\\
	&\leq \frac{d}{s^2}\|g\|_\infty^2,
	\end{align}
	where the inequality holds due to \eqref{helpppp_3}.
\end{proof}

\section{Proof of \cref{lem:var_single}}\label{app:lem_var_single}
\begin{proof}
	For any $g$, as we know that $\expect[\phi^{\prime}(g)|g]=g$, the variance 
	of 
	$\phi^{\prime}(g)$ can be written as 
	\begin{align}\label{kashk}
	\var[\phi^{\prime}(g)|g]
	&= \expect[\|\phi^{\prime}(g)-g\|^2|g] \nonumber \\
	&= \sum_{i=1}^d\ \expect[(g_i - \sign(g_i)b_i\|g\|_\infty)^2|g] \nonumber \\
	&= \sum_{i=1}^d\ g_i^2 +\|g\|_\infty^2 \expect[b_i^2|g] -2 g_i 
	\sign(g_i)|g_i|\nonumber \\
	&= \sum_{i=1}^d\ \|g\|_\infty^2 \expect[b_i^2|g] -g_i^2,
	\end{align}	
	where the third equality follows from $\expect[b_i|g]=|g_i|/\|g\|_\infty$ 
	and $\sign(g_i)^2=1$. Note that based on the probability distribution of 
	$b_i$, we can simplify the expression $\expect[b_i^2|g]$ as 
	${|g_i|}/{\|g\|_\infty}$ and write 
	\begin{equation*}			
	\var[\phi^{\prime}(g)|g] 
	= \sum_{i=1}^d [\|g\|_\infty|g_i|-g_i^2] 
	= \|g\|_1\|g\|_\infty - \|g\|_2^2,
	\end{equation*}
	which shows that the claim in \eqref{lem:var_single_claim} holds.
\end{proof}

\section{Bounding $\|\nabla f(x_t)-\bar{g}_t\|$ in Stochastic 
	Case}\label{app:var_reduction_stochastic}
In order to upper bound $\|\nabla f(x_t)-\bar{g}_t\|$, we need a lemma for 
variance reduction, which is a generalization of Lemma 2 in 
\citep{mokhtari2017conditional}.

\begin{lemma}
	\label{lem:new_variance_reduction}
	Let $ \{ 
	a_t\}_{t=0}^{T}$ be a sequence of points in $\mathbb{R}^n$ 
	such that $ \| a_t - 
	a_{t-1} \| 
	\leq G/(t+s)^\alpha $ for all $1\leq t\leq T $, where constants $ G 
	\geq 0$, $\alpha \in (0,1]$, $s \geq 
	8^{\frac{1}{2\alpha}}-1$. 
	Let $ \{ \tilde{a}_t\}_{t=1}^T$ be a sequence of random variables such 
	that $ \expect[ 
	\tilde{a}_t|\mathcal{F}_{t-1} ] = a_t $ and $ \expect[ \| \tilde{a}_t - 
	a_t 
	\|^2|\mathcal{F}_{t-1} ] \leq \sigma^2$ for 
	every $ t\geq 1 $, 
	where 
	$ 
	\mathcal{F}_{t-1} $ is the $ \sigma $-field generated by 
	$ \{ \tilde{a}_i\}_{i=1}^{t-1} $ 
	and $ \mathcal{F}_{0} = \varnothing $. Let $\{d_t\}_{t=0}^T$ be a 
	sequence of random 
	variables where $d_0$ is fixed and subsequent $d_{t}$ are obtained by 
	the recurrence 
	\begin{equation*}
	d_t = (1-\rho_t) d_{t-1} +\rho_t \tilde{a}_t
	\end{equation*}
	with $ \rho_t = \frac{2}{(t+s)^{2\alpha/3}} $. 
	Then we have
	\begin{equation*}
	\expect[\| a_t-d_t\|^2 ] \leq \frac{Q}{(t+s+1)^{2\alpha/3}}
	\end{equation*}
	where $ Q \triangleq \max \{ \|a_0 - d_0 \|^2 (s+1)^{2\alpha/3}, 
	4\sigma^2 + 2G^2 \} $.
\end{lemma}
\begin{proof}
	First, for all $t\geq 1$, we have $\rho_t \geq 0$ and
	\begin{equation*}
	\rho_t \leq \frac{2}{(1+s)^{2\alpha/3}} \leq 
	\frac{2}{(8^{\frac{1}{2\alpha}})^{2\alpha/3}}=1.
	\end{equation*}		
	
	Then we define $\Delta_t=\|a_t-d_t\|^2$, thus
	\begin{equation*}
	\begin{split}
	\expect[\Delta_t|\mathcal{F}_{t-1}]
	={}& \expect[\|\rho_t(a_t-\tilde{a}_t)+(1-\rho_t)(a_t-a_{t-1}) 
	+(1-\rho_t)(a_{t-1}-d_{t-1}) \|^2|\mathcal{F}_{t-1}]   \\
	\leq{}& \rho_t^2\sigma^2 
	+(1-\rho_t)^2\frac{G^2}{(t+s)^{2\alpha}}+(1-\rho_t)^2\Delta_{t-1} \\
	&\quad +2(1-\rho_t)^2\expect[\langle 
	a_t-a_{t-1},a_{t-1}-d_{t-1}\rangle|\mathcal{F}_{t-1}].
	\end{split}    
	\end{equation*}	
	
	By Law of Total Expectation,
	\begin{equation*}
	\begin{split}
	\expect[\Delta_t]
	={}&\expect[\expect[\Delta_t|\mathcal{F}_{t-1}]]\\
	\leq{} & \rho_t^2\sigma^2 
	+(1-\rho_t)^2\frac{G^2}{(t+s)^{2\alpha}}+(1-\rho_t)^2\expect[\Delta_{t-1}]
	+2(1-\rho_t)^2\expect[\langle 
	a_t-a_{t-1},a_{t-1}-d_{t-1}\rangle].
	\end{split}    
	\end{equation*}	
	
	Apply Young's inequality, we have
	\begin{equation*}
	2\langle a_t-a_{t-1},a_{t-1}-d_{t-1}\rangle 
	\leq{} 
	\beta_t\|a_{t-1}-d_{t-1} \|^2+\frac{\|a_t-a_{t-1} 
		\|^2}{\beta_t}	
	\leq{}
	\beta_t\|a_{t-1}-d_{t-1} \|^2+\frac{G^2}{\beta_t(t+s)^{2\alpha}}\,.
	\end{equation*}	
	
	So if we let $z_t = \expect[\Delta_t]$ and set $\beta_t = \rho_t/2$, we 
	have
	\begin{equation*}
	\begin{split}
	z_t &\leq \rho_t^2\sigma^2 
	+(1-\rho_t)^2\frac{G^2}{(t+s)^{2\alpha}}+(1-\rho_t)^2z_{t-1}
	+(1-\rho_t)^2(\beta_tz_{t-1}+\frac{G^2}{\beta_t(t+s)^{2\alpha}}) \\
	&\leq \rho_t^2\sigma^2 
	+(1-\rho_t)^2(1+1/\beta_t)\frac{G^2}{(t+s)^{2\alpha}} + 
	(1-\rho_t)^2(1+\beta_t)z_{t-1} \\
	&=\rho_t^2\sigma^2 +(1-\rho_t)^2(1+2/\rho_t)\frac{G^2}{(t+s)^{2\alpha}} 
	+ (1-\rho_t)^2(1+\rho_t/2)z_{t-1} \\
	&\leq \rho_t^2\sigma^2 +(1+2/\rho_t)\frac{G^2}{(t+s)^{2\alpha}} + 
	(1-\rho_t)z_{t-1}.
	\end{split}    
	\end{equation*}
	
	The last inequality holds since $\rho_t \in [0,1]$ 
	implies $(1-\rho_t)^2 \leq 1$ and $(1-\rho_t)(1+\rho_t/2) \leq 1$. Now 
	we can further simplify $z_t$
	\begin{equation*}
	\begin{split}
	z_t \leq{} & 
	(1-\frac{2}{(t+s)^{2\alpha/3}})z_{t-1}+\frac{4\sigma^2}{(t+s)^{4\alpha/3}} 
	+\frac{G^2}{(t+s)^{2\alpha}}+ \frac{G^2}{(t+s)^{4\alpha/3}}\\
	\leq{} & (1-\frac{2}{(t+s)^{2\alpha/3}})z_{t-1} + 
	\frac{4\sigma^2+2G^2}{(t+s)^{4\alpha/3}} \\
	\leq{} & (1-\frac{2}{(t+s)^{2\alpha/3}}) z_{t-1} + 
	\frac{Q}{(t+s)^{4\alpha/3}}.
	\end{split}    
	\end{equation*}
	
	Now we claim that $z_t \leq \frac{Q}{(t+s+1)^{2\alpha/3}}$ for 
	$\forall\ 
	t \in \{0, 1, \cdots, T\}$. We show it by induction. The statement 
	holds for $t=0$ because of the definition of $Q$. If the statement 
	holds for some $t=k-1$, where $k \in [T]$, then
	\begin{equation*}
	\begin{split}
	z_k \leq{} & (1-\frac{2}{(k+s)^{2\alpha/3}}) 
	z_{k-1}+\frac{Q}{(k+s)^{4\alpha/3}}\\
	\leq{} & (1-\frac{2}{(k+s)^{2\alpha/3}}) 
	\frac{Q}{(k+s)^{2\alpha/3}}+\frac{Q}{(k+s)^{4\alpha/3}}\\
	={}&\frac{(k+s)^{2\alpha/3}-1}{(k+s)^{4\alpha/3}}Q.
	\end{split}    
	\end{equation*}
	
	So we only need to prove that 
	\begin{equation*}
	\frac{(k+s)^{2\alpha/3}-1}{(k+s)^{4\alpha/3}} \leq 
	\frac{1}{(k+s+1)^{2\alpha/3}}    
	\end{equation*}
	or equivalently,
	\begin{equation*}
	[(k+s)^{2\alpha/3}-1]\cdot (k+s+1)^{2\alpha/3} \leq 
	(k+s)^{4\alpha/3}.    
	\end{equation*}
	
	It suffices to show that
	\begin{equation*}
	(k+s+1)^{2\alpha/3} \leq (k+s)^{2\alpha/3}+1.    
	\end{equation*}
	
	Consider $f(k) = (k+s+1)^{2\alpha/3} - (k+s)^{2\alpha/3}-1$, we observe 
	that $f(-s)=0$, and 
	$\frac{\mathrm{d}f(k)}{\mathrm{d}k}=\frac{2\alpha}{3}(k+s+1)^{2\alpha/3-1}
	-\frac{2\alpha}{3}(k+s)^{2\alpha/3-1}
	< 0$ for $\forall\ k > -s$ since $\alpha \in (0,1]$. So $f(k)$ a 
	decreasing function on $[-s, \infty)$, thus $f(k) \leq 0$ for $\forall\ 
	k \in \mathbb{N}^+$, which implies $(k+s+1)^{2\alpha/3} \leq 
	(k+s)^{2\alpha/3}+1$.
	
	So the statement holds for $t=k$, and by induction, $z_t \leq 
	\frac{Q}{(t+s+1)^{2\alpha/3}}$ for $\forall\ t \in \{0, 1, \cdots, T\}$.
\end{proof}

\section{Proof of \cref{thm:distributed}}\label{app:thm_distributed}
\begin{proof}
	First, since $x_{t+1} = (1-\eta_t)x_t + \eta_t v_t$ is a convex combination 
	of $x_t, v_t$, and $x_1 \in \constraint, v_t \in \constraint, \forall\ t$, 
	we can prove $x_t \in \constraint, \forall\ t $ by induction. So $x_{T+1} 
	\in \constraint$.
	
	Then we observe that for any iteration $t$, we have
	\begin{equation}
	\label{eq:iteration}
	\begin{split}
	f(x_{t+1})-f(x^*)
	={}&f(x_t + \eta_t(v_t-x_t))-f(x^*) \\
	\stackrel{(a)}{\leq}& f(x_t)+\eta_t \langle v_t-x_t, \nabla f(x_t) \rangle 
	+\frac{L}{2}\eta_t^2\|v_t-x_t\|^2 -f(x^*) \\
	={}& f(x_t)-f(x^*) +\eta_t \langle v_t-x_t, \bar{g}_t \rangle + \eta_t 
	\langle 
	v_t-x_t, \nabla f(x_t)-\bar{g}_t \rangle 
	+\frac{L}{2}\eta_t^2\|v_t-x_t\|^2 \\
	\stackrel{(b)}{\leq} & f(x_t)-f(x^*) +\eta_t \langle x^*-x_t, \bar{g}_t 
	\rangle  + \eta_t \langle v_t-x_t, \nabla f(x_t)-\bar{g}_t \rangle 
	+\frac{L}{2}\eta_t^2\|v_t-x_t\|^2 \\
	={}& f(x_t)-f(x^*) +\eta_t \langle x^*-x_t, \nabla f(x_t) \rangle 
	+\frac{L}{2}\eta_t^2\|v_t-x_t\|^2 \\
	&\quad + \eta_t \langle v_t-x_t, \nabla 
	f(x_t)-\bar{g}_t \rangle +\eta_t 
	\langle x^*-x_t, \bar{g}_t-\nabla f(x_t)\rangle\\
	\stackrel{(c)}{\leq} & (1-\eta_t)(f(x_t)-f(x^*)) + \eta_t \langle v_t-x^*, 
	\nabla f(x_t)-\bar{g}_t 
	\rangle+\frac{L}{2}\eta_t^2\|v_t-x_t\|^2 \\
	\stackrel{(d)}{\leq} & (1-\eta_t)(f(x_t)-f(x^*)) + \eta_t \|v_t-x^*\|\cdot 
	\|\nabla 
	f(x_t)-\bar{g}_t\|+\frac{L}{2}\eta_t^2\|v_t-x_t\|^2 \\
	\leq{} & (1-\eta_t)(f(x_t)-f(x^*)) + \eta_t D\|\nabla 
	f(x_t)-\bar{g}_t\|+\frac{L}{2}\eta_t^2D^2
	\end{split}
	\end{equation}
	where Inequality $(a)$ holds because of the $L$-smoothness. In $(b)$ we 
	used the optimality of $v_t$. Inequality $(c)$ is due to the convexity of 
	$f$, and we applied the Cauchy-Schwarz inequality in $(d)$. 
	

	Now we want to apply \cref{lem:new_variance_reduction} to bound $\|\nabla 
	f(x_t)-\bar{g}_t\|$. By the smoothness of $f$ and $\eta_t = \frac{2}{t+3}$, 
	we have 
	\begin{equation*}
	\| \nabla f(x_{t}) - \nabla f(x_{t-1}) \| \leq \eta_{t-1}\|v_{t-1}-x_{t-1} 
	\| L 
	\leq \frac{2LD}{t+2}. 
	\end{equation*}
	
	Let $ \mathcal{F}_{t-1} $ be the $ \sigma $-field generated by 
	$ \{ \Phi^{\prime}(\tilde{g}_i)\}_{i=1}^{t-1} $ 
	and $ \mathcal{F}_{0} = \varnothing $, by the unbiasedness of the random 
	encoding scheme $\Phi$ and the 
	stochastic gradient $g_t^m$, we have for all $t \geq 1$
	\begin{equation*}
	\expect[\Phi^{\prime}(\tilde{g}_t)|\mathcal{F}_{t-1}] = \expect 
	[\tilde{g}_t|\mathcal{F}_{t-1}] 
	= \expect[\frac{\sum_{m=1}^M 
		\Phi^{\prime}(g_t^m(x_t))}{M}|\mathcal{F}_{t-1}] 
	= \expect[\frac{\sum_{m=1}^M g_t^m(x_t)}{M}|\mathcal{F}_{t-1}] 
	= \nabla f(x_t)
	\end{equation*}
	and
	\begin{equation*}
	\begin{split}
	&\expect[\|\Phi^{\prime}(\tilde{g}_t)-\nabla f(x_t)\|^2|\mathcal{F}_{t-1}] 
	\\
	={}& \expect[\|\Phi^{\prime}(\tilde{g}_t)-\tilde{g}_t + \tilde{g}_t - 
	\frac{\sum_{m=1}^M g_t^m(x_t)}{M} + \frac{\sum_{m=1}^M g_t^m(x_t)}{M} 
	-\nabla f(x_t)\|^2|\mathcal{F}_{t-1}] \\
	={}& \expect[\|\Phi^{\prime}(\tilde{g}_t)-\tilde{g}_t + \frac{\sum_{m=1}^M 
		\Phi^{\prime}(g_t^m(x_t))}{M} - \frac{\sum_{m=1}^M g_t^m(x_t)}{M} + 
	\frac{\sum_{m=1}^M g_t^m(x_t)}{M} -\nabla 
	f(x_t)\|^2|\mathcal{F}_{t-1}] \\
	={}& \expect[\|\frac{\sum_{m=1}^M \Phi^{\prime}(g_t^m(x_t))}{M} - 
	\frac{\sum_{m=1}^M g_t^m(x_t)}{M}\|^2|\mathcal{F}_{t-1}] + 
	\expect[\|\frac{\sum_{m=1}^M g_t^m(x_t)}{M} -\nabla 
	f(x_t)\|^2|\mathcal{F}_{t-1}]  \\
	&\quad 
	+\expect[\|\Phi^{\prime}(\tilde{g}_t)-\tilde{g}_t\|^2|\mathcal{F}_{t-1}] + 
	2\expect[\langle 
	\Phi^{\prime}(\tilde{g}_t)-\tilde{g}_t,\frac{\sum_{m=1}^M 
		\Phi^{\prime}(g_t^m(x_t))}{M} - \frac{\sum_{m=1}^M g_t^m(x_t)}{M} 
	\rangle|\mathcal{F}_{t-1}] \\
	&\quad + 2\expect[\langle 
	\Phi^{\prime}(\tilde{g}_t)-\tilde{g}_t,\frac{\sum_{m=1}^M g_t^m(x_t)}{M} 
	-\nabla f(x_t) \rangle|\mathcal{F}_{t-1}] \\
	& \quad + 2\expect[\langle \frac{\sum_{m=1}^M \Phi^{\prime}(g_t^m(x_t))}{M} 
	- 
	\frac{\sum_{m=1}^M g_t^m(x_t)}{M}, \frac{\sum_{m=1}^M g_t^m(x_t)}{M} 
	-\nabla 
	f(x_t)\rangle|\mathcal{F}_{t-1}].
	\end{split}
	\end{equation*}
	
	By \cref{assump_on_estimate,assump_on_quantization}, we have
	\begin{align*}
	\expect[\|\frac{\sum_{m=1}^M \Phi^{\prime}(g_t^m(x_t))}{M} - 
	\frac{\sum_{m=1}^M g_t^m(x_t)}{M}\|^2|\mathcal{F}_{t-1}]
	={} &
	\frac{\sum_{m=1}^M\expect[\|\Phi^{\prime}(g_t^m(x_t)) - 
		g_t^m(x_t)\|^2|\mathcal{F}_{t-1}]}{M^2} 
	\leq \frac{\sigma_2^2}{M}\,, \\
%
%
	\expect[\|\frac{\sum_{m=1}^M g_t^m(x_t)}{M} -\nabla 
	f(x_t)\|^2|\mathcal{F}_{t-1}]
	={} & \frac{\sum_{m=1}^M 
		\var[g_t^m(x_t)|\mathcal{F}_{t-1}]}{M^2} 
	\leq 
	\frac{\sigma_1^2}{M}\,,\\  
%
	\expect[\|\Phi^{\prime}(\tilde{g}_t)-\tilde{g}_t\|^2|\mathcal{F}_{t-1}]
	\leq{} & \sigma_3^2  \,, 
	\end{align*}
	and
	\begin{align*}
	\expect[\langle \Phi^{\prime}(\tilde{g}_t)-\tilde{g}_t,\frac{\sum_{m=1}^M 
		\Phi^{\prime}(g_t^m(x_t))}{M} - \frac{\sum_{m=1}^M g_t^m(x_t)}{M} 
		\rangle 
	|\mathcal{F}_{t-1}] ={} & 0\,,\\  
%
	\expect[\langle \Phi^{\prime}(\tilde{g}_t)-\tilde{g}_t,\frac{\sum_{m=1}^M 
		g_t^m(x_t)}{M} -\nabla f(x_t) \rangle|\mathcal{F}_{t-1}] ={} & 0\,,  
		\\  
%
	\expect[\langle \frac{\sum_{m=1}^M \Phi^{\prime}(g_t^m(x_t))}{M} - 
	\frac{\sum_{m=1}^M g_t^m(x_t)}{M}, \frac{\sum_{m=1}^M g_t^m(x_t)}{M} 
	-\nabla 
	f(x_t)\rangle|\mathcal{F}_{t-1}]={} & 0\,.
	\end{align*}
	
	Therefore,
	\begin{equation}
	\label{eq:variance_bound}
	\expect[\|\Phi^{\prime}(\tilde{g}_t)-\nabla f(x_t)\|^2|\mathcal{F}_{t-1}] 
	\leq 
	\frac{\sigma_2^2}{M} + \frac{\sigma_1^2}{M} + \sigma_3^2 
	= \frac{\sigma_1^2+\sigma_2^2+M\sigma_3^2}{M}\,.
	\end{equation}
	
	Now apply \cref{lem:new_variance_reduction} with $\alpha=1, G=2LD, s=2> 
	2\sqrt{2}-1
	=8^{\frac{1}{2\alpha}}-1, \sigma^2 
	= \frac{\sigma_1^2+\sigma_2^2+M\sigma_3^2}{M}, d_t = \bar{g}_t, 
	\forall\ t \geq 0, a_t = \nabla f(x_t), \tilde{a}_t = 
	\Phi^\prime(\tilde{g}_t), \forall\ t \geq 1, a_0 = a_1 = \nabla f(x_1)$, we 
	have
	\begin{equation*}
	\expect[\|\nabla f(x_t) - \bar{g}_t\|^2] \leq \frac{Q}{(t+3)^{2/3}}\,,
	\end{equation*}
	where $Q = \max \{3^{2/3}\|\nabla f(x_1) \|^2, 
	\frac{4(\sigma_1^2+\sigma_2^2)}{M}+4\sigma_3^2+8L^2D^2 \}$.
	
	By Jensen's Inequality, 
	\begin{equation}
	\label{eq:variance_reduction} 
	\expect[\|\nabla f(x_t) - \bar{g}_t\|] \leq \sqrt{\expect[\|\nabla f(x_t) - 
		\bar{g}_t\|^2]} \leq \frac{Q^{1/2}}{(t+3)^{1/3}}.
	\end{equation}
	
	Combine \cref{eq:iteration,eq:variance_reduction} and recall $\eta_t = 
	\frac{2}{t+3}$, we have
	\begin{equation}
	\label{eq:iter}
	\begin{split}
	\expect[f(x_{t+1})]-f(x^*)
	\leq{}& (1-\eta_t)(\expect[f(x_t)]-f(x^*)) + \eta_t 
	\frac{DQ^{1/2}}{(t+3)^{1/3}}+\frac{L}{2}\eta_t^2D^2  \\
	\leq{}& 
	(1-\frac{2}{t+3})(\expect[f(x_t)]-f(x^*))+\frac{2D(Q^{1/2}+LD)}{(t+3)^{4/3}}.
	\end{split}
	\end{equation}
	
	Now we claim that for all $t \in [T+1]$
	\begin{equation*}
	\expect[f(x_{t})]-f(x^*) \leq \frac{Q_0}{(t+3)^{1/3}}    
	\end{equation*}
	where $Q_0=\max\{4^{1/3}\cdot2M_0, 2D(Q^{1/2}+LD)\}$.
	
	We prove it by induction. When $t=1$, we have 
	\begin{equation*}
	\frac{Q_0}{(t+3)^{1/3}} \geq \frac{4^{1/3}\cdot2M_0}{4^{1/3}}=2M_0 \geq 
	\expect[f(x_{1})]-f(x^*).   
	\end{equation*}
	
	Now suppose that for some $t \in [T]$, we have $\expect[f(x_{t})]-f(x^*) 
	\leq 
	\frac{Q_0}{(t+3)^{1/3}}$, then by \cref{eq:iter}, we have 
	
	\begin{equation*}
	\begin{split}
	\expect[f(x_{t+1})]-f(x^*) &\leq (1-\frac{2}{t+3}) 
	\frac{Q_0}{(t+3)^{1/3}}+\frac{Q_0}{(t+3)^{4/3}}
	= \frac{Q_0}{(t+3)^{1/3}}-\frac{Q_0}{(t+3)^{4/3}} \\
	&= \frac{(t+2)Q_0}{(t+3)^{4/3}} 
	\leq \frac{Q_0}{(t+4)^{1/3}}
	\end{split}
	\end{equation*}
	where the last inequality holds since $(t+2)^3(t+4) \leq (t+3)^4, \forall\ 
	t 
	\geq 1$. Therefore, we have
	\begin{equation*}
	\expect[f(x_{t})]-f(x^*) \leq \frac{Q_0}{(t+3)^{1/3}}, \forall\ t \in 
	[T+1].    
	\end{equation*}
	
	Specifically, we have
	\begin{equation*}
	\expect[f(x_{T+1})]-f(x^*) \leq \frac{Q_0}{(T+4)^{1/3}}.   
	\end{equation*}
\end{proof}

\section{Proof of Theorem 1 with \sscheme}\label{app:thm_sign}

Next, we incorporate the \sscheme into \sqfw as quantization scheme. We first 
make the 
following
assumption on the stochastic gradients.

\begin{assump}
	\label{assump_for_encoding}
	The stochastic gradients $g_t^m$ have uniformly bounded $\ell_1$ and 
	$\ell_\infty$ norms, \emph{i.e.}, $\|g_t^m\|_1 \leq G_1, \|g_t^m\|_\infty 
	\leq 
	G_\infty, \forall\ m \in [M], t \in [T].$   
\end{assump}


\begin{corollary}
	\label{thm:sign}
	Under 
	\cref{assump_on_K,assump_on_f,assump_on_estimate,assump_for_encoding}, if 
	we set $\eta_t 
	\!=\! 2/(t+3), \rho_t \!=\! 2/(t+3)^{2/3}$ and apply  \sscheme 
	 in S-QFW, then
	$\expect[f(x_{T+1})]\!-\!f(x^*)\! \leq\! \frac{Q_0}{(T+4)^{1/3}}$,
	where $Q_0=\max\{4M_0, 2D(Q^{1/2}\!+\!LD)\}$, and $Q = \max 
	\{3\|\nabla f(x_1)\|^2, 
	4(\sigma_1^2+G_1G_\infty)/M\!+\!4G_1G_\infty\!+\!8L^2D^2 \}$.
\end{corollary}

\begin{proof}
	Since $\sign(g)\circ b$ requires $2d$ bits and 
	$\|g\|_\infty$ requires 32 
	bits, so for each $\phi(g)$, we need $2d+32$ bits of communication. 
	At Step 4 of \Alg, each worker $m$ should push $\phi(g_t^m)$ to the master, 
	and at Step 6, the master should broadcast $\phi(\tilde{g}_t)$ to all the 
	$M$ workers, so we need $(2d+32)\cdot 
	M+2d+32=(M+1)(2d+32)$ bits per round.
	
	In order to apply \cref{thm:distributed}, we only need to prove that $\phi$ 
	has similar properties to \cref{assump_on_quantization}. We have shown that 
	the \sscheme $\phi$ is unbiased. Then by \cref{lem:var_single}, we 
	have
	\begin{equation*}
	\expect[\|\phi^{\prime}(g_t^m) - g_t^m\|^2] = 
	\expect[\expect[\|\phi^{\prime}(g_t^m) - g_t^m\|^2|g_t^m]] =
	\expect[\|g_t^m\|_1\|g_t^m\|_\infty - 
	\|g_t^m\|_2^2] \leq G_1G_\infty.   
	\end{equation*}
	
	Since
	\begin{align*}
	\expect[\|\phi^{\prime}(\tilde{g}_t)-\tilde{g}_t\|^2]=
	\expect[\expect[\|\phi^{\prime}(\tilde{g}_t)-\tilde{g}_t\|^2|\tilde{g}_t]]
	&= \expect[\|\tilde{g}_t\|_1\|\tilde{g}_t\|_\infty - 
	\|\tilde{g}_t\|_2^2]  
	\leq  G_\infty\expect[\|\tilde{g}_t\|_1]
	\end{align*}
	and
	\begin{equation*}
	\begin{split}
	\expect[\|\tilde{g}_t\|_1|g_t^m] &= \expect[\|\frac{\sum_{m=1}^M 
		\phi^{\prime}(g_t^m)}{M} \|_1|g_t^m] 
	\leq \frac{\sum_{m=1}^M \expect[\|\phi^{\prime}(g_t^m)\|_1|g_t^m]}{M}\\ 
	&=\frac{\sum_{m=1}^M \expect[\sum_{i=1}^d 
		|\phi_i^{\prime}(g_t^m)||g_t^m]}{M} 
	= \frac{\sum_{m=1}^M \sum_{i=1}^d |g_{t,i}^m|}{M} \\
	&= \frac{\sum_{m=1}^M \|g_{t}^m\|_1}{M} 
	\leq G_1,
	\end{split}
	\end{equation*}
	where $\phi_i^{\prime}(g_t^m)$ is the $i$th element of 
	$\phi^{\prime}(g_t^m)$, $g_{t,i}^m$ is the $i$th element of $g_{t}^m$. So 
	we have
	\begin{align*}
	\expect[\|\tilde{g}_t\|_1] = \expect[\expect[\|\tilde{g}_t\|_1|g_t^m]] \le 
	G_1,
	\end{align*}
	and
	\begin{equation*}
	\expect[\|\phi^{\prime}(\tilde{g}_t)-\tilde{g}_t\|^2] \leq G_1G_\infty. 
	\end{equation*}
	
	we can apply \cref{thm:distributed} with $\sigma_2^2 = G_1G_\infty, 
	\sigma_3^2 = G_1G_\infty$, then we have
	\begin{equation*}
	\expect[f(x_{T+1})]-f(x^*) \leq \frac{Q_0}{(T+4)^{1/3}}     
	\end{equation*}
	where $Q_0=\max\{4^{1/3}\cdot2M_0, 2D(Q^{1/2}+LD)\}$, and $Q=\max 
	\{3^{2/3}\|\nabla f(x_1) \|^2, 
	\frac{4(\sigma_1^2+G_1G_\infty)}{M}+4G_1G_\infty+8L^2D^2 \}$.
	
	
\end{proof}

\section{Proof of 
	\cref{thm:distributed_nonconvex}}\label{app:thm_distributed_nonconvex}

\begin{proof}
	First, since $x_{t+1} = (1-\eta_t)x_t + \eta_t v_t$ is a convex combination 
	of $x_t, v_t$, and $x_1 \in \constraint, v_t \in \constraint, \forall\ t$, 
	we can prove $x_t \in \constraint, \forall\ t $ by induction. So the output 
	$x_o \in \constraint.$
	
	Note that if we define $v_t^\prime = \argmin_{v \in \constraint}\langle v, 
	\nabla f(x_t)\rangle$, then $\mathcal{G}(x_t)=\langle 
	v^\prime_t-x_t,-\nabla f(x_t)\rangle = -\langle v^\prime_t-x_t,\nabla 
	f(x_t)\rangle$. So we have
	
	\begin{equation*}
	\begin{split}
	f(x_{t+1}) 
	\stackrel{(a)}{\leq}& f(x_t) + \langle \nabla f(x_t),x_{t+1}-x_t \rangle + 
	\frac{L}{2}\|x_{t+1}-x_t\|^2 \\
	={}&f(x_t) + \langle \nabla f(x_t),\eta_t(v_t-x_t) \rangle + 
	\frac{L}{2}\|\eta_t(v_t-x_t)\|^2 \\
	\stackrel{(b)}{\leq}& f(x_t) + \eta_t \langle \nabla f(x_t),v_t-x_t 
	\rangle+\frac{L\eta_t^2D^2}{2} \\
	={}& f(x_t) + \eta_t \langle  \bar{g}_t,v_t-x_t \rangle+ \eta_t \langle 
	\nabla 
	f(x_t)-\bar{g}_t,v_t-x_t \rangle  + \frac{L\eta_t^2D^2}{2} \\
	\stackrel{(c)}{\leq}& f(x_t) + \eta_t \langle  \bar{g}_t,v^\prime_t-x_t 
	\rangle+ \eta_t \langle \nabla f(x_t)-\bar{g}_t,v_t-x_t \rangle + 
	\frac{L\eta_t^2D^2}{2} \\
	={}& f(x_t) + \eta_t \langle \nabla f(x_t),v^\prime_t-x_t \rangle  + \eta_t 
	\langle  \bar{g}_t-\nabla f(x_t),v^\prime_t-x_t \rangle \\
	&\quad + \eta_t \langle 
	\nabla f(x_t)-\bar{g}_t,v_t-x_t \rangle + 
	\frac{L\eta_t^2D^2}{2} \\
	={}& f(x_t) - \eta_t \mathcal{G}(x_t) + \eta_t \langle \nabla 
	f(x_t)-\bar{g}_t,v_t-v^\prime_t \rangle + \frac{L\eta_t^2D^2}{2} \\
	\stackrel{(d)}{\leq}& f(x_t) - \eta_t \mathcal{G}(x_t) + \eta_t \|\nabla 
	f(x_t)-\bar{g}_t\|\| v_t-v^\prime_t\| + \frac{L\eta_t^2D^2}{2} \\
	\stackrel{(e)}{\leq}& f(x_t) - \eta_t \mathcal{G}(x_t) + \eta_tD \|\nabla 
	f(x_t)-\bar{g}_t\| + \frac{L\eta_t^2D^2}{2},
	\end{split}
	\end{equation*}
	where we used the assumption that $f$ has $L$-Lipschitz continuous gradient 
	in inequality (a). Inequalities (b), (e) hold because of 
	\cref{assump_on_K}. Inequality (c) is due to the optimality of $v_t$, and 
	in (d), we applied the Cauchy-Schwarz inequality.
	
	Rearrange the inequality above, we have 
	\begin{equation}
	\label{eq:bound_on_individual_gap}
	\eta_t \mathcal{G}(x_t) \leq f(x_t)- f(x_{t+1})+ \eta_tD \|\nabla 
	f(x_t)-\bar{g}_t\| + \frac{L\eta_t^2D^2}{2}.   
	\end{equation}
	
	Apply \cref{eq:bound_on_individual_gap} recursively for $t = 1, 2, \cdots, 
	T$, and take expectations, we attain the following inequality:
	\begin{equation}
	\label{eq:bound_on_gap}
	\begin{split}
	\sum_{t=1}^T \eta_t \expect[\mathcal{G}(x_t)]
	\leq f(x_1)-f(x_{T+1})
	+D\sum_{t=1}^T \eta_t\expect[\|\nabla f(x_t)-\bar{g}_t\|] 
	+\frac{LD^2}{2}\sum_{t=1}^T \eta_t^2.
	\end{split}
	\end{equation}
	
	Since $f$ has $L$-Lipschitz continuous gradient, and $\eta_t=(T+3)^{-3/4}$, 
	we have 
	\begin{align*}
	\| \nabla f(x_{t}) - \nabla f(x_{t-1}) \|\leq{} & 
	\|\eta_{t-1}(v_{t-1}-x_{t-1}) \| 	L 
	={} \eta_{t-1}\|v_{t-1}-x_{t-1} \| L\\
	\leq{}& \frac{LD}{(T+3)^{3/4}} 
	\leq{} \frac{LD}{(t+3)^{3/4}}. 
	\end{align*}
	
	Combine the inequality above with \cref{eq:variance_bound}, and apply 
	\cref{lem:new_variance_reduction} with $\alpha=3/4,G=LD, 
	s=3=8^{\frac{1}{2\cdot3/4}}-1=8^{\frac{1}{2\alpha}}-1, \sigma^2 = 
	\frac{\sigma_1^2+\sigma_2^2+M\sigma_3^2}{M}, d_t = \bar{g}_t, 
	\forall\ t \geq 0, a_t = \nabla f(x_t), \tilde{a}_t = 
	\Phi^\prime(\tilde{g}_t), \forall\ t \geq 1, a_0 = a_1 = \nabla f(x_1)$, we 
	have
	\begin{equation*}
	\expect[\|\nabla f(x_t) - \bar{g}_t\|^2] \leq \frac{Q}{(t+4)^{1/2}}
	\end{equation*}
	where $Q = \max \{2\|\nabla f(x_1)\|^2, 
	\frac{4(\sigma_1^2+\sigma_2^2)}{M}+4\sigma_3^2+2L^2D^2 \}$.
	
	By Jensen's Inequality, 
	\begin{equation*}
	\expect[\|\nabla f(x_t) - \bar{g}_t\|] \leq \sqrt{\expect[\|\nabla f(x_t) - 
		\bar{g}_t\|^2]} \leq \frac{Q^{1/2}}{(t+4)^{1/4}}.
	\end{equation*}
	
	Since 
	\begin{equation*}
	\begin{split}
	\sum_{t=1}^T\expect[\|\nabla f(x_t) - \bar{g}_t\|]&\leq 
	\sum_{t=1}^T\frac{Q^{1/2}}{(t+4)^{1/4}}  
	\leq \int_{0}^T 
	\frac{Q^{1/2}}{(t+4)^{1/4}}\mathrm{d}t \\
	&=\frac{4Q^{1/2}}{3}[(T+4)^{3/4}-4^{3/4}] 
	\leq 
	\frac{4Q^{1/2}}{3}(T+4)^{3/4} \,,   
	\end{split}
	\end{equation*}
	by \cref{eq:bound_on_gap}, we have
	\begin{equation*}
	\begin{split}
	\sum_{t=1}^T\expect[\mathcal{G}(x_t)]
	\leq{}& 
	\frac{f(x_1)-f(x_{T+1})}{(T+3)^{-3/4}}+D\sum_{t=1}^T\expect[\|\nabla 
	f(x_t) - \bar{g}_t\|] +\frac{LD^2}{2}\frac{T\cdot 
		(T+3)^{-3/2}}{(T+3)^{-3/4}}  	\\
	\leq{}& 
	[f(x_1)-f(x_{T+1})](T+3)^{3/4}+\frac{4DQ^{1/2}}{3}(T+4)^{3/4} 
	+\frac{LD^2}{2}T(T+3)^{-3/4}.
	\end{split}
	\end{equation*}
	
	So we have
	\begin{equation*}
	\begin{split}
	\expect[\mathcal{G}(x_o)] 
	={}&\frac{\sum_{t=1}^T\expect[\mathcal{G}(x_t)]}{T}  \\
	\leq{}& 
	[f(x_1)-f(x_{T+1})]\frac{(T+3)^{3/4}}{T}+\frac{4DQ^{1/2}}{3}
	\frac{(T+4)^{3/4}}{T} +\frac{LD^2}{2}(T+3)^{-3/4}\\
	\leq{}& 2M_0 
	\frac{(T+3)^{3/4}}{(T+3)/4}+\frac{4DQ^{1/2}}{3}\frac{(T+4)^{3/4}}{(T+4)/5}
	+\frac{LD^2}{2}(T+3)^{-3/4} \\
	\leq{}& \frac{8M_0+20DQ^{1/2}/3}{(T+3)^{1/4}}+\frac{LD^2}{2(T+3)^{3/4}}.
	\end{split}
	\end{equation*}
	
\end{proof}

\section{Bounding $\| \nabla f(x_t) - \bar{g}_t\|$ in Finite-Sum 
	Case}\label{app:var_reduction_finite}
We now address the bound of $\| \nabla f(x_t) - \bar{g}_t\|$, which is 
resolved in the following lemma.


\begin{lemma}
	\label{lem:finite_variance_reduction}
	Under \cref{assump_on_K}, if we further assume that each $f_{m,i}$ is 
	bounded and has 
	$L$-Lipschitz continuous gradient, and $\|\nabla f_{m,i}(x)\|_\infty \leq 
	G_\infty, \forall\ x \in \constraint, m \in [M], i \in [n]$, where 
	$G_\infty$ is a positive constant, by setting $p=\sqrt{n}, S=\sqrt{n}$ and 
	applying the $s_{1,t}=(2^{z_{1,t}}-1)$-\mscheme $\phi_{1,t}$, the 
	$s_{2,t}=(2^{z_{2,t}}-1)$-\mscheme $\phi_{2,t}$ as $\Phi_{1,t}, 
	\Phi_{2,t}$, we have
	\begin{align*}
	\expect[\| \nabla f(x_t) - \bar{g}_t\|^2] \leq 
	(L^2D^2+2G_\infty^2)\frac{\sum_{\lfloor \frac{t-1}{p}\rfloor p+1 \leq l 
			\leq (\lfloor \frac{t-1}{p}\rfloor +1)p}\eta_l^2}{\sqrt{n}},    
	\end{align*}
	where 
	\begin{equation}\label{eq:lemma_z1}
	z_{1,t}=\lceil\log_2[(\frac{4dpS}{M\sum_{\lfloor 
			\frac{t-1}{p}\rfloor 
			p+1 \leq l \leq (\lfloor \frac{t-1}{p}\rfloor +1)p} 
		\eta_l^2})^{1/2}+1]\rceil,
	\end{equation}
	\begin{equation}\label{eq:lemma_z2}
	z_{2,t}=\lceil\log_2[(\frac{4dpS}{\sum_{\lfloor \frac{t-1}{p}\rfloor p+1 
			\leq l \leq (\lfloor \frac{t-1}{p}\rfloor +1)p} 
		\eta_l^2})^{1/2}+1]\rceil.
	\end{equation}
	
\end{lemma}

This lemma looks a bit complicated because of the summation of 
$\eta_l^2$. The range of the summation is just the subset which contains $l$ 
and can be expressed as $\{kp+1, kp+2, \cdots, (k+1)p\}$, where $k \in 
\mathbb{N}$. This is easy to understand, since intuitively, the variance is 
only related to the factors within the same period $\{kp+1, kp+2, \cdots, 
(k+1)p\}$. In practical applications, we usually have concrete values for 
$\eta_l$, 
which will make the sum and thus the expressions look much simpler. 


\begin{proof}
	We first define an auxiliary variable $g_t$, which is $\nabla f(x_t)$ if 
	$\text{mod}(t,p)=1$, and is set to $\frac{1}{MS}\sum_{m=1}^M\sum_{i \in 
		\mathcal{S}_{t}^m}[\nabla 
	f_{m,i}(x_t)-\nabla f_{m,i}(x_{t-1})] 
	+ g_{t-1}$ otherwise.
	
	We also define $\mathcal{F}_{t-1}$ to be the $\sigma$-field generated by 
	all the randomness before round $t$. We note that given 
	$\mathcal{F}_{t-1}$, $x_t$ is actually determined, and we can verify that 
	$\expect[g_t|\mathcal{F}_{t-1}]=\nabla f(x_t)$, and 
	$\expect[\bar{g}_t|\mathcal{F}_{t-1},g_t]=g_t, \forall\ t \in [T]$. Here, 
	with abuse of notation, $\expect[\cdot|g_t]$ is the conditional expectation 
	given not only the value of $g_t$, but also the gradients $\nabla 
	f_{m,i}(x_t), \nabla f_{m,i}(x_{t-1})$ for all $i \in \mathcal{S}_{t}^m, m 
	\in [M]$.

	Then by 
	law of total expectation, we have
	\begin{equation}
	\label{eq:finite_nonconvex_bound}
	\begin{split}
	\expect[\| \nabla f(x_t) - \bar{g}_t\|^2]
	={}&\expect[\expect[\| \nabla f(x_t) - \bar{g}_t\|^2|\mathcal{F}_{t-1}]] \\
	={}& \expect[\expect[\| \nabla f(x_t) - g_t+ 
	g_t-\bar{g}_t\|^2|\mathcal{F}_{t-1}]]    \\
	={}& \expect[\expect[\| \nabla f(x_t) - g_t\|^2|\mathcal{F}_{t-1}]] + 
	\expect[\expect[\| g_t -\bar{g}_t\|^2|\mathcal{F}_{t-1}]] \\
	&\quad + 
	2\expect[\expect[\langle \nabla f(x_t) - g_t, g_t -\bar{g}_t \rangle 
	|\mathcal{F}_{t-1}]]\\
	={}& \expect[\| \nabla f(x_t) - g_t\|^2] + \expect[\| g_t -\bar{g}_t\|^2],
	\end{split}    
	\end{equation}
	where the last equation holds since 
	\begin{equation*}
	\begin{split}
	\expect[\langle \nabla f(x_t) - g_t, g_t -\bar{g}_t 
	\rangle|\mathcal{F}_{t-1}] 
	={}& \expect[\expect[\langle \nabla f(x_t) - g_t, g_t -\bar{g}_t 
	\rangle|\mathcal{F}_{t-1},g_t]|\mathcal{F}_{t-1}] \\
	={}& \expect[\langle \nabla f(x_t) - g_t, \expect[g_t 
	-\bar{g}_t|\mathcal{F}_{t-1},g_t] \rangle|\mathcal{F}_{t-1}] \\
	={}&0.
	\end{split}    
	\end{equation*}
	
	Moreover, for $\text{mod}(t,p)\neq 1$,
	\begin{equation*}
	\begin{split}
	\expect[\| \nabla f(x_t) - g_t\|^2] 
	={}& \expect[\expect[\| \nabla f(x_t) -\nabla f(x_{t-1})+\nabla 
	f(x_{t-1})-g_{t-1}\\
	& - \frac{1}{MS}\sum_{m=1}^M\sum_{i \in 
		\mathcal{S}_{t}^m}[\nabla 
	f_{m,i}(x_t)-\nabla f_{m,i}(x_{t-1})]\|^2|\mathcal{F}_{t-1}]]. 
	\end{split}    
	\end{equation*}
	
	With abuse of notation, we have $\expect[\sum_{m=1}^M \nabla 
	f_{m,i}(x_t)/M|\mathcal{F}_{t-1}] = \nabla f(x_t)$, and 
	\[\expect[\sum_{m=1}^M\nabla f_{m,i}(x_{t-1})/M|\mathcal{F}_{t-1}]=\nabla 
	f(x_{t-1})\,,
	\] where $i$ actually depends on $m$, and is sampled from 
	$\mathcal{S}_{t}^m$ at random. Thus
	\begin{align*}
	\expect[\frac{1}{MS}\sum_{m=1}^M\sum_{i \in \mathcal{S}_{t}^m}[\nabla 
	f_{m,i}(x_t)-\nabla f_{m,i}(x_{t-1})]|\mathcal{F}_{t-1}] 
	=  \nabla f(x_t) -\nabla f(x_{t-1}),  
	\end{align*}
	and
	\begin{align*}
	\expect[\langle \nabla f(x_t) -\nabla f(x_{t-1}) - 
	\frac{1}{MS}\sum_{m=1}^M\sum_{i \in \mathcal{S}_{t}^m} [\nabla 
	f_{m,i}(x_t)-\nabla f_{m,i}(x_{t-1})], \nabla 
	f(x_{t-1})-g_{t-1}\rangle|\mathcal{F}_{t-1}]=0.    
	\end{align*}
	
	So we have
	\begin{equation*}
	\begin{split}
	&\expect[\| \nabla f(x_t) - g_t\|^2] \\
	={}& \expect[\expect[\| \frac{1}{MS}\sum_{m=1}^M\sum_{i \in 
		\mathcal{S}_{t}^m}[\nabla f_{m,i}(x_t)-\nabla 
	f_{m,i}(x_{t-1})] - [\nabla f(x_t) 
	-\nabla f(x_{t-1})] \|^2 |\mathcal{F}_{t-1}]] \\
	&  + \expect[\expect[\|\nabla 
	f(x_{t-1})-g_{t-1}\|^2|\mathcal{F}_{t-1}]]  \\	
	={}& \expect[\var[\frac{1}{MS}\sum_{m=1}^M\sum_{i \in 
		\mathcal{S}_{t}^m}[\nabla f_{m,i}(x_t)-\nabla 
	f_{m,i}(x_{t-1})]|\mathcal{F}_{t-1}]] \\
	& + \expect[\expect[\|\nabla 
	f(x_{t-1})-g_{t-1}\|^2|\mathcal{F}_{t-1}]]  \\
	={}&\frac{1}{S}\expect[\var[\frac{\sum_{m=1}^M \nabla f_{m,i}(x_t)-\nabla 
		f_{m,i}(x_{t-1})}{M}|\mathcal{F}_{t-1}]] + \expect[\|\nabla 
	f(x_{t-1})-g_{t-1}\|^2]\\
	\leq{} & \frac{1}{S}\expect[\expect[\|\frac{\sum_{m=1}^M \nabla 
		f_{m,i}(x_t)-\nabla f_{m,i}(x_{t-1})}{M}\|^2|\mathcal{F}_{t-1}]] + 
	\expect[\|\nabla f(x_{t-1})-g_{t-1}\|^2]\\
	\leq{} & \frac{1}{S}\expect[\expect[(\frac{\sum_{m=1}^M \|\nabla 
		f_{m,i}(x_t)-\nabla f_{m,i}(x_{t-1})\|}{M})^2|\mathcal{F}_{t-1}]] + 
	\expect[\|\nabla f(x_{t-1})-g_{t-1}\|^2]\\
	\leq{} & \frac{1}{S}(LD\eta_t)^2+ \expect[\|\nabla f(x_{t-1})-g_{t-1}\|^2]
	= \frac{L^2D^2\eta_t^2}{S}+ \expect[\|\nabla f(x_{t-1})-g_{t-1}\|^2].
	\end{split}    
	\end{equation*}
	
	Note for any $t$ such that $\text{mod}(t,p)=1$, we have $g_t=\nabla 
	f(x_t)$. Therefore
	\begin{equation}
	\label{eq:finite_nonconvex_bound_1}
	\expect[\| \nabla f(x_t) - g_t\|^2] \leq \frac{L^2D^2}{S}\sum_{\lfloor 
		\frac{t-1}{p}\rfloor p+1 \leq l \leq (\lfloor \frac{t-1}{p}\rfloor 
		+1)p} 
	\eta_l^2.  
	\end{equation}
	
	Now we turn to bound $\expect[\| g_t -\bar{g}_t\|^2]$. For $\text{mod}(t,p) 
	\neq 1$, We have
	\begin{equation*}
	\begin{split}
	&\expect[\| g_t -\bar{g}_t\|^2] \\
	={}& \expect[\expect[\| \frac{1}{MS}\sum_{m=1}^M\sum_{i \in 
		\mathcal{S}_{t}^m}[\nabla f_{m,i}(x_t)-\nabla f_{m,i}(x_{t-1})] + 
		g_{t-1}
	-\phi_{2,t}^\prime(\tilde{g}_t)-\bar{g}_{t-1}\|^2|\mathcal{F}_{t-1},g_t]] \\
	={}& \expect[\expect[\| \frac{1}{MS}\sum_{m=1}^M\sum_{i \in 
		\mathcal{S}_{t}^m}[\nabla f_{m,i}(x_t)-\nabla 
	f_{m,i}(x_{t-1})] - 
	\phi_{2,t}^\prime(\tilde{g}_t)\|^2|\mathcal{F}_{t-1},g_t]] \\
	& + \expect[\expect[\|g_{t-1} 
	-\bar{g}_{t-1}\|^2|\mathcal{F}_{t-1},g_t]] \\
	& + 2\expect[\expect[\langle \frac{1}{MS}\sum_{m=1}^M\sum_{i \in 
		\mathcal{S}_{t}^m}[\nabla f_{m,i}(x_t)-\nabla 
	f_{m,i}(x_{t-1})] -\phi_{2,t}^\prime(\tilde{g}_t), g_{t-1} -\bar{g}_{t-1} 
	\rangle|\mathcal{F}_{t-1},g_t]].
	\end{split}    
	\end{equation*}
	
	Moreover
	\begin{equation*}
	\begin{split}
	\expect[\phi_{2,t}^\prime(\tilde{g}_t)|\mathcal{F}_{t-1},g_t] 
	={}& 
	\expect[\tilde{g}_t|\mathcal{F}_{t-1},g_t]\\
	={}&\expect[\sum_{m=1}^M 
	\phi_{1,t}^\prime (\frac{\sum_{i \in \mathcal{S}_{t}^m}\nabla 
		f_{m,i}(x_t)-\nabla f_{m,i}(x_{t-1})}{S})/M |\mathcal{F}_{t-1},g_t] \\
	={}& \frac{1}{MS}\sum_{m=1}^M\sum_{i \in \mathcal{S}_{t}^m}[\nabla 
	f_{m,i}(x_t)-\nabla f_{m,i}(x_{t-1})],
	\end{split}
	\end{equation*}
	and
	\begin{equation*}
	\begin{split}
	& \expect[\expect[\| \frac{1}{MS}\sum_{m=1}^M\sum_{i \in 
		\mathcal{S}_{t}^m}[\nabla f_{m,i}(x_t)-\nabla 
	f_{m,i}(x_{t-1})] 
	-\phi_{2,t}^\prime(\tilde{g}_t)\|^2|\mathcal{F}_{t-1},g_t]] \\    
	={}& \expect[\expect[\| \frac{1}{MS}\sum_{m=1}^M\sum_{i \in 
		\mathcal{S}_{t}^m}[\nabla f_{m,i}(x_t)-\nabla f_{m,i}(x_{t-1})] 
	-\tilde{g}_t 
	+ \tilde{g}_t -\phi_{2,t}^\prime(\tilde{g}_t)\|^2|\mathcal{F}_{t-1},g_t]] \\
	={}& \expect[\expect[\| \frac{1}{MS}\sum_{m=1}^M\sum_{i \in 
		\mathcal{S}_{t}^m}[\nabla f_{m,i}(x_t)-\nabla f_{m,i}(x_{t-1})]  \\
	&  -\sum_{m=1}^M \phi_{1,t}^\prime (\frac{\sum_{i \in 
			\mathcal{S}_{t}^m}\nabla f_{m,i}(x_t)- \nabla 
			f_{m,i}(x_{t-1})}{S})/M 
	\|^2 |\mathcal{F}_{t-1},g_t]] \\
	& + 
	\expect[\expect[\|\tilde{g}_t-\phi_{2,t}^\prime(\tilde{g}_t)\|^2|\mathcal{F}_{t-1},g_t,\tilde{g}_t]]
	\\
	\leq{} & \frac{1}{M}\frac{d}{s_{1,t}^2}(2G_\infty)^2 + 
	\frac{d}{s_{2,t}^2}(2G_\infty)^2 
	= \frac{4dG_\infty^2}{Ms_{1,t}^2}+\frac{4dG_\infty^2}{s_{2,t}^2},
	\end{split}    
	\end{equation*}
	where in the inequality, we apply \cref{lem:var_multi} with 
	$\|\frac{\sum_{i \in \mathcal{S}_{t}^m}\nabla f_{m,i}(x_t)- \nabla 
		f_{m,i}(x_{t-1})}{S} \|_\infty \leq 2G_\infty$ and 
		$\|\tilde{g}_t\|_\infty 
	= \|\sum_{m=1}^M \phi_{1,t}^\prime (\frac{\sum_{i \in 
			\mathcal{S}_{t}^m}\nabla f_{m,i}(x_t)- \nabla 
		f_{m,i}(x_{t-1})}{S})/M\|_\infty \leq 2G_\infty$.
	
	Now we have for $\text{mod}(t,p) \neq 1$,
	\begin{equation*}
	\expect[\| g_t -\bar{g}_t\|^2] \leq 
	\frac{4dG_\infty^2}{Ms_{1,t}^2}+\frac{4dG_\infty^2}{s_{2,t}^2} + 
	\expect[\|g_{t-1} -\bar{g}_{t-1}\|^2].
	\end{equation*}
	
	If $\text{mod}(t,p)=1$, we have
	\begin{equation*}
	\begin{split}
	\expect[\| g_t -\bar{g}_t\|^2] 
	={}& \expect[\|\nabla f(x_t) -\tilde{g}_t+\tilde{g}_t- 
	\phi_{2,t}^\prime(\tilde{g}_t)\|^2] \\
	={}& \expect[\expect[\|\nabla f(x_t) - \frac{\sum_{m=1}^M 
		\phi_{1,t}^\prime(\sum_{i=1}^n \nabla 
		f_{m,i}(x_t)/n))}{M}\|^2 |\mathcal{F}_{t-1},g_t]] \\
	&\quad + \expect[\expect[\| 
	\tilde{g}_t- 
	\phi_{2,t}^\prime(\tilde{g}_t)\|^2|\mathcal{F}_{t-1},g_t,\tilde{g}_t]] \\
	\leq{}& \frac{1}{M} \expect[\expect[\|\sum_{i=1}^n \nabla f_{m,i}(x_t)/n -  
	\phi_{1,t}^\prime(\sum_{i=1}^n \nabla 
	f_{m,i}(x_t)/n))\|^2 |\mathcal{F}_{t-1},g_t]] + 
	\frac{d}{s_{2,t}^2}G_\infty^2 \\
	\leq{} & \frac{dG_\infty^2}{Ms_{1,t}^2}+\frac{dG_\infty^2}{s_{2,t}^2},
	\end{split}    
	\end{equation*}
	where in the inequality, we apply \cref{lem:var_multi} with $\|\sum_{i=1}^n 
	\nabla f_{m,i}(x_t)/n \|_\infty \leq G_\infty$ and $\|\tilde{g}_t\|_\infty 
	= \|\frac{\sum_{m=1}^M \phi_{1,t}^\prime(\sum_{i=1}^n \nabla 
		f_{m,i}(x_t)/n))}{M}\|_\infty \leq G_\infty$. Since for any $t_1, t_2$ 
		such 
	that $\lfloor \frac{t_1-1}{p} \rfloor = \lfloor \frac{t_2-1}{p} \rfloor$, 
	we have $s_1 \triangleq s_{1,t_1}=s_{1,t_2}, s_2 \triangleq 
	s_{2,t_1}=s_{2,t_2}$, thus
	\begin{equation}
	\label{eq:finite_nonconvex_bound_2} 
	\begin{split}
	\expect[\| g_t -\bar{g}_t\|^2] \leq 
	[\frac{4dG_\infty^2}{Ms_1^2}+\frac{4dG_\infty^2}{s_2^2}](p-1) + 
	\frac{dG_\infty^2}{Ms_1^2}+\frac{dG_\infty^2}{s_2^2}
	\leq 
	\frac{4dpG_\infty^2}{Ms_1^2}+\frac{4dpG_\infty^2}{s_2^2}.
	\end{split}  	
	\end{equation}
	
	Now combine 
	\cref{eq:finite_nonconvex_bound,eq:finite_nonconvex_bound_1,eq:finite_nonconvex_bound_2},
	we have
	\begin{align*}
	\expect[\| \nabla f(x_t) - \bar{g}_t\|^2] \leq
	\frac{L^2D^2}{S}\sum_{\lfloor \frac{t-1}{p}\rfloor p+1 \leq l \leq (\lfloor 
		\frac{t-1}{p}\rfloor +1)p} \eta_l^2+ 
	\frac{4dpG_\infty^2}{Ms_1^2}+\frac{4dpG_\infty^2}{s_2^2}.   
	\end{align*}
	
	Since we set $p=\sqrt{n}, S=\sqrt{n}, s_1=2^{z_1}-1\geq 
	(\frac{4dpS}{M\sum_{\lfloor \frac{t-1}{p}\rfloor p+1 \leq l \leq (\lfloor 
			\frac{t-1}{p}\rfloor +1)p}\eta_l^2})^{1/2}, s_2=2^{z_2}-1\geq 
	(\frac{4dpS}{\sum_{\lfloor \frac{t-1}{p}\rfloor p+1 \leq l \leq (\lfloor 
			\frac{t-1}{p}\rfloor +1)p}\eta_l^2})^{1/2}$, we have
	\begin{align*}
	\expect[\| \nabla f(x_t) - \bar{g}_t\|^2] \leq
	(L^2D^2+2G_\infty^2)\frac{\sum_{\lfloor \frac{t-1}{p}\rfloor p+1 \leq l 
			\leq (\lfloor \frac{t-1}{p}\rfloor +1)p}\eta_l^2}{\sqrt{n}}.    
	\end{align*}
\end{proof}

\section{Proof of \cref{thm:finite_convex}}\label{app:thm_finite_convex}

\begin{proof}
	First, since $x_{t+1} = (1-\eta_t)x_t + \eta_t v_t$ is a convex combination 
	of $x_t, v_t$, and $x_1 \in \constraint, v_t \in \constraint, \forall\ t$, 
	we can prove $x_t \in \constraint, \forall\ t $ by induction. So $x_{T+1} 
	\in \constraint$.
	
	Since for the $s$-\sscheme, the required number of bits is $z=\log_2(s+1)$. 
	So 
	for $\Phi_{1,t}, \Phi_{2,t}$, the corresponding assigned bits are 
	$z_{1,t}=\log_2(s_{1,t}+1), z_{2,t}=\log_2(s_{2,t}+1)$. To make them 
	integers, we can set $z_{1,t}=\lceil 
	\log_2[(\frac{pd^{1/2}S^{1/2}}{M^{1/2}})\lceil 
	\frac{t}{p}\rceil +1] \rceil, z_{2,t} = \lceil 
	\log_2[pd^{1/2}S^{1/2}\lceil\frac{t}{p}\rceil + 1] \rceil$.
	
%
%
	
	Then with $\eta_t=2/(p\lceil \frac{t}{p}\rceil)$, the conditions 
	\cref{eq:lemma_z1,eq:lemma_z2} in \cref{lem:finite_variance_reduction} are 
	satisfied. So by \cref{lem:finite_variance_reduction}, we have
	\begin{equation*}
	\expect[\| \nabla f(x_t)-\bar{g}_t \|^2] \leq 
	\frac{4(L^2D^2+2G_\infty^2)}{p^2\lceil \frac{t}{p}\rceil^2}.   
	\end{equation*}
	
	So
	\begin{align*}
	\expect[\| \nabla f(x_t)-\bar{g}_t \|] \leq  \sqrt{\expect[\| \nabla 
		f(x_t)-\bar{g}_t \|^2]} 
	\leq \frac{2\sqrt{L^2D^2+2G_\infty^2}}{p\lceil 
		\frac{t}{p}\rceil}.
	\end{align*}
	
	On the other hand, by \cref{assump_on_f_i_convex}, $f=\frac{\sum_{m \in 
			[M], i \in [n]}f_{m,i}}{Mn}$ is a bounded $L$-smooth convex 
			function on 
	$\constraint$, with 
	$\sup_{x \in \constraint}|f(x)|\leq M_0$. So \cref{eq:iteration} still 
	holds. Taking expectation on both sides, we have
	\begin{equation*}
	\begin{split}
	\expect[f(x_{t+1})] - f(x^*) 
	\leq{}& (1-\eta_t)(\expect[f(x_t)]-f(x^*)) + 
	\eta_t D\expect[\|\nabla f(x_t)-\bar{g}_t\|] +\frac{L}{2}\eta_t^2D^2 \\   
	\leq{}& (1-\frac{2}{p\lceil \frac{t}{p}\rceil})(\expect[f(x_t)]-f(x^*)) + 
	\frac{2D(2\sqrt{L^2D^2+2G_\infty^2}+LD)}{p^2\lceil \frac{t}{p}\rceil^2}.
	\end{split}
	\end{equation*}
	
	Let $t = kp, Q=2D(2\sqrt{L^2D^2+2G_\infty^2}+LD)$ where $k \in 
	\mathbb{N}^+$, and apply the inequality recursively for $p$ times, we have
	\begin{align*}
	\expect[f(x_{kp+1})] - f(x^*) \leq (1-\frac{2}{pk})^p 
	(\expect[f(x_{(k-1)p+1})]-f(x^*)) + \frac{Q}{pk^2}.
	\end{align*}
	
	Now we claim that $(1-\frac{2}{pk})^p \leq 1-\frac{2}{pk}p + 
	\frac{4}{p^2k^2}\frac{p(p-1)}{2}=1-\frac{2}{k}+\frac{2(p-1)}{pk^2}.$ 
	
	The inequality holds trivially for $p=1$ and $p=2$. For $p \geq 3$, we have 
	$\frac{2}{pk} < 1$. Define function 
	$h(x)=(1-x)^p-1+px-\frac{p(p-1)}{2}x^2$. Then for $x \in [0,1]$, we have 
	$h^\prime(x)=p[1-(p-1)x-(1-x)^{p-1}]  
	,h^{\prime\prime}(x)=p(p-1)[(1-x)^{p-2}-1] \leq 0$, then $h^\prime(x)\leq 
	h^\prime(0)=0$. Thus $h(x) \leq h(0)=0$, \emph{i.e.}, $(1-x)^p \leq 
	1-px+\frac{p(p-1)}{2}x^2$. Let $x=\frac{2}{pk}$, then we have 
	$(1-\frac{2}{pk})^p \leq 
	1-p\frac{2}{pk}+\frac{p(p-1)}{2}(\frac{2}{pk})^2=1-\frac{2}{k}+\frac{2(p-1)}{pk^2}.$
	
	Consider $k \geq 3$, then 
	\begin{align*}
	(1-\frac{2}{pk})^p \leq 1 - \frac{2}{k} + \frac{2(p-1)}{3pk} 
	\leq 1 - \frac{2}{k} + \frac{2}{3k} 
	= 1 - \frac{4}{3k},
	\end{align*}
	and thus
	\begin{align*}
	\expect[f(x_{kp+1})] - f(x^*) \leq (1-\frac{4}{3k}) 
	(\expect[f(x_{(k-1)p+1})]-f(x^*)) + \frac{Q}{pk^2}.    
	\end{align*}
	
	Define $Q_0=\max\{6pM_0, 3Q\}$. Then we claim $\expect[f(x_{kp+1})] - 
	f(x^*) 
	\leq \frac{Q_0}{(k+1)p}, \forall\ k \in \mathbb{N}$. 
	
	We prove this inequality by induction. For $k=0,1,2$, we have 
	$\expect[f(x_{kp+1})] - f(x^*) \leq 2M_0 \leq \frac{Q_0}{(2+1)p} \leq 
	\frac{Q_0}{(k+1)p}.$ Now suppose that for some $k\geq 3$, we have 
	$\expect[f(x_{(k-1)p+1})]-f(x^*)\leq \frac{Q_0}{kp}$, then
	\begin{equation*}
	\begin{split}
	\expect[f(x_{kp+1})] - f(x^*) \leq (1-\frac{4}{3k}) \frac{Q_0}{kp} + 
	\frac{Q_0}{3pk^2}   
	=\frac{k-1}{pk^2}Q_0 
	\leq \frac{Q_0}{(k+1)p},
	\end{split}    
	\end{equation*}
	where the last inequality holds since $(k-1)(k+1)\leq k^2$. 
	
	So we have $\expect[f(x_{kp+1})] - f(x^*) \leq \frac{Q_0}{(k+1)p}$, for all 
	non-negative integer $k \leq T/p$. Let $T=Kp$, then
	\begin{align*}
	\expect[f(x_{T+1})] - f(x^*) = \expect[f(x_{Kp+1})] - f(x^*)
	 \leq 
	\frac{Q_0}{(K+1)p} 
	\leq \frac{Q_0}{T}.   
	\end{align*}
	
	For any $\epsilon > 0$, set $T=\frac{Q_0}{\epsilon}$, then we have 
	$\expect[f(x_{T+1})] - f(x^*)] \leq \epsilon$. So the LO complexity is 
	$\mathcal{O}(1/\epsilon)$. Also note in each period, the total number of 
	gradient 
	call 
	is $Mn+(p-1)\cdot M\cdot S\cdot 2=Mn+2MS(p-1)$, so the average cost is 
	$[Mn+2MS(p-1)]/p=M[3\sqrt{n}-2]$. Thus the total IFO complexity is 
	$M[3\sqrt{n}-2]\frac{Q_0}{\epsilon} 
	=\mathcal{O}(\sqrt{n}\max\{6M_0\sqrt{n},3Q 
	\}/\epsilon)$.
	
	The communication bits per round are at most 
	$M[d(z_{1,T}+1)+32]+d(z_{2,T}+1)+32=d(Mz_{1,T}+z_{2,T})+(M+1)(d+32)
	\approx d(M\lceil\log_2[(\frac{\sqrt{n}dT^2}{M})^{1/2}+1]\rceil+
	\lceil\log_2[(\sqrt{n}dT^2)^{1/2}+1]\rceil) + (M+1)(d+32).$
	
\end{proof}

\section{Proof of 
	\cref{thm:quantized_finite_distributed_nonconvex}}
\label{app:thm_quantized_finite_distributed_nonconvex}
\begin{lemma}
	\label{lem:finite_distributed_nonconvex}
	Under \cref{assump_on_K,assump_on_f_i}, with $\eta_t = T^{-1/2}$ and fixed 
	$T$ in \cref{alg:finite_dist_fw}, if we further assume that
	\begin{equation*}
	\expect[\| \nabla f(x_t) - \bar{g}_t\|^2] \leq \frac{c^2}{T},    
	\end{equation*}
	where $c$ is a positive constant, then we have $x_o \in \constraint$ and
	\begin{equation*}
	\expect[\mathcal{G}(x_o)] \leq \frac{2M_0+cD+\frac{LD^2}{2}}{\sqrt{T}}.
	\end{equation*}
\end{lemma}

Since here we already assume that $\expect[\| \nabla f(x_t) - \bar{g}_t\|^2]$ 
has an upper bound. The convergence rate can be proved by solving a recursive 
inequality directly. Moreover, since finite-sum optimization is a special case 
of stochastic gradient optimization, we can use the analysis in the proof of 
\cref{thm:distributed_nonconvex} to get the inequality.

\begin{proof}
	First, since $x_{t+1} = (1-\eta_t)x_t + \eta_t v_t$ is a convex combination 
	of $x_t, v_t$, and $x_1 \in \constraint, v_t \in \constraint, \forall\ t$, 
	we can prove $x_t \in \constraint, \forall\ t $ by induction. So $x_{o} \in 
	\constraint$.
	
	By \cref{assump_on_f_i}, $f$ is also a bounded (potentially) non-convex 
	function on $\constraint$ with $L$-Lipschitz continuous gradient. 
	Specifically, we have $\sup_{x \in \constraint}|f(x)| \leq M_0$. So 
	\cref{eq:bound_on_gap} still holds, \emph{i.e.},
	\begin{align*}
	\sum_{t=1}^T \eta_t \expect[\mathcal{G}(x_t)] \leq
	f(x_1)-f(x_{T+1})+D\sum_{t=1}^T \eta_t\expect[\|\nabla f(x_t)-\bar{g}_t\|] 
	+\frac{LD^2}{2}\sum_{t=1}^T \eta_t^2.    
	\end{align*}
	
	Since we assume that $\expect[\| \nabla f(x_t) - \bar{g}_t\|^2] \leq 
	\frac{c^2}{T}$, we have
	\begin{equation*}
	\expect[\|\nabla f(x_t)-\bar{g}_t\|] \leq \sqrt{\expect[\|\nabla 
		f(x_t)-\bar{g}_t\|^2]} \leq \frac{c}{\sqrt{T}}.    
	\end{equation*}
	
	With $\eta_t = T^{-1/2}$, we then have
	\begin{equation*}
	\begin{split}
	\sum_{t=1}^T \expect[\mathcal{G}(x_t)]
	\leq{}& \sqrt{T}[f(x_1)-f(x_{T+1})] + 
	D \sum_{t=1}^T \expect[\|\nabla f(x_t)-\bar{g}_t\|] + \sqrt{T} 
	\frac{LD^2}{2} T (T^{-1/2})^2 \\
	\leq{}& 2M_0\sqrt{T}+DT\frac{c}{\sqrt{T}}+\frac{LD^2}{2}\sqrt{T}
	= (2M_0+cD+\frac{LD^2}{2})\sqrt{T}.
	\end{split}
	\end{equation*}
	
	So
	\begin{equation*}
	\expect[\mathcal{G}(x_o)]=\frac{\sum_{t=1}^T \expect[\mathcal{G}(x_t)]}{T} 
	\leq \frac{2M_0+cD+\frac{LD^2}{2}}{\sqrt{T}}.    
	\end{equation*}
\end{proof}

Now we can prove \cref{thm:quantized_finite_distributed_nonconvex}.

\begin{proof}
	By \cref{lem:finite_distributed_nonconvex}, we only need to bound 
	$\expect[\| \nabla f(x_t) - \bar{g}_t\|^2]$, which can be achieved by 
	applying \cref{lem:finite_variance_reduction}. 
	
	Since for the $s$-\sscheme, the required number of bits is $z=\log_2(s+1)$. 
	So for $\Phi_{1,t}, \Phi_{2,t}$, the corresponding assigned bits are 
	$z_{1,t}=\log_2(s_{1,t}+1), z_{2,t}=\log_2(s_{2,t}+1)$. To make them 
	integers, we can set  $z_{1,t} = z_1 
	= \lceil \log_2[(\frac{4\sqrt{n}dT}{M})^{1/2} + 1] \rceil, z_{2,t}=z_2 = 
	\lceil \log_2[(4\sqrt{n}dT)^{1/2} +1] \rceil$. 
	
	
	
	Then with $\eta_t=T^{-1/2}$, the conditions \cref{eq:lemma_z1,eq:lemma_z2} 
	in \cref{lem:finite_variance_reduction} are satisfied. So by 
	\cref{lem:finite_variance_reduction}, we have
	\begin{equation*}
	\begin{split}
	\expect[\| \nabla f(x_t) - \bar{g}_t\|^2]
	\leq{}
	(L^2D^2+2G_\infty^2)\frac{\sum_{\lfloor \frac{t-1}{p}\rfloor p+1 \leq l 
			\leq (\lfloor \frac{t-1}{p}\rfloor
			+1)p}\eta_l^2}{\sqrt{n}} 
	={}\frac{L^2D^2+2G_\infty^2}{T}.
	\end{split}    
	\end{equation*}

	By \cref{lem:finite_distributed_nonconvex},
	\begin{equation*}
	\expect[\mathcal{G}(x_o)] \leq 
	\frac{2M_0+D\sqrt{L^2D^2+2G_\infty^2}+\frac{LD^2}{2}}{\sqrt{T}}.  
	\end{equation*}
	
	The average communication bits per round are 
	$M[d(z_1+1)+32]+d(z_2+1)+32=d(Mz_1+z_2)+(M+1)(d+32).$
	
	For any $\epsilon > 0$, set 
	$T=(2M_0+D\sqrt{L^2D^2+2G_\infty^2}+\frac{LD^2}{2})^2/\epsilon^2$, then we 
	have $\expect[\mathcal{G}(x_o)] \leq \epsilon$. So the LO complexity is 
	$\mathcal{O}(1/\epsilon^2)$. Also note in each period, the total number of 
	gradient 
	call is $Mn+(p-1)\cdot M\cdot S\cdot 2=Mn+2MS(p-1)$, so the average cost is 
	$[Mn+2MS(p-1)]/p=M[3\sqrt{n}-2]$. Thus the total IFO complexity is 
	$M[3\sqrt{n}-2](2M_0+D\sqrt{L^2D^2+2G_\infty^2}+\frac{LD^2}{2})^2/\epsilon^2
	=\mathcal{O}(\sqrt{n}/\epsilon^2)$.

\end{proof}

\section{Additional Experiment Results}\label{app:additonal_experiments}
\input{additional_experiments}

\end{document}

%% file: challenges_2_app.tex

\def \thisplotscale {0.8}
\def \unit {1 cm}

%
\tikzstyle{block}         = [ draw,
                              rectangle, rounded corners,
                              minimum height = 0.8*\unit,
                              minimum width  = 1.1*\unit,
                              text width     = 3.1*\unit,
                              text badly centered,
                              line width=1pt,
                              fill = yellow!30, 
                              font = \footnotesize, 
                              anchor = west]
\tikzstyle{bold block}    = [ block,
                              fill = blue!40]
\tikzstyle{light block}   = [ block,
                              fill = blue!10]
\tikzstyle{connector}     = [ draw, 
                              -stealth, 
                              shorten >=2,
                              shorten <=2,]
\tikzstyle{dot dot dot}   = [ draw, 
                              dashed]

\tikzstyle{blockbig}         = [ draw,
                              rectangle, 
                               minimum height = 0.8*\unit,
                              minimum width  = 1.1*\unit,
                              text width     = 3.1*\unit,
                              text badly centered,
                               line width=1pt,
                              fill = green!30, 
                              font = \small, 
                              anchor = west]

\begin{tikzpicture}[scale=0.8]

\node at (-2,15.3) {\textbf{1st stage}};

\path (3,15.2)                 node [blockbig] (C0) {Master}; 

\path (-3,13.5)                 node [block] (C1) {W$_1$: \black{Compute $g_t^1(x_t)$}}; 

\path (3,13.5)                 node [block] (C2) {W$_m$: \black{Compute $g_t^m(x_t)$}}; 
\path (9,13.5)                 node [block] (C3) {W$_M$: \black{Compute $g_t^M(x_t)$}}; 

\node at (1.8,13.5) {\ldots};
\node at (7.8,13.5) {\ldots};

\draw [dashed] (-3,12.7) -- (13,12.7);

\draw [->,line width=1pt, black] (C1.north) -- (C0.west)  node[midway,above,left] {$\black{\Phi(g_t^1(x_t))}$} ; 
\draw [->,line width=1pt, black]  (C2.north) -- (C0.south) node[midway,left] {$\black{\Phi(g_t^m(x_t))}$}; 
\draw [->,line width=1pt, black] (C3.north) -- (C0.east)   node[midway,right] {$\black{\Phi(g_t^M(x_t))}$}; 

\node at (-2,12) {\textbf{2nd stage}};

\path (3,11.6)                 node [blockbig] (C10) {Master:\\ \black{$\displaystyle{\tilde{g}_t \gets \sum_{m=1}^M \frac{\Phi^{\prime}(g_t^m(x_t))}{M}}$}}; 

\path (-3,9.4)                 node [block] (C11) {W$_1$}; 
\path (3,9.4)                 node [block] (C12) {W$_m$}; 
\path (9,9.4)                 node [block] (C13) {W$_M$}; 

\node at (1.8,9.4) {\ldots};
\node at (7.8,9.4) {\ldots};

\draw [->,line width= 1pt, black]  (C10.west) -- (C11.north) node[midway,above,left] {$\black{\Phi(\tilde{g}_t)}$} ; 
\draw [->,line width=1pt, black]  (C10.south) -- (C12.north)  node[midway,left] {$\black{\Phi(\tilde{g}_t)}$}; 
\draw [->,line width=1pt, black]  (C10.east)  -- (C13.north)  node[midway,right] {$\black{\Phi(\tilde{g}_t)}$}; 
\draw [dashed] (-3,8.6) -- (13,8.6);

\node at (-2,8) {\textbf{3rd stage}};

\path (3,7.85)                 node [blockbig] (C90) {Master}; 

\path (-3,6.2)                 node [block] (C91)  {W$_1$:\\
		\black{$\bar{g}_t \!\gets \!VR(\bar{g}_{t-1} ,\Phi^{\prime}(\tilde{g}_t))$ $x_{t+1} \gets FW(x_t,\bar{g}_t )$}}; 
\path (3,6.2)                 node [block] (C92) {W$_m$:\\
		\black{$\bar{g}_t \!\gets \!VR(\bar{g}_{t-1} ,\Phi^{\prime}(\tilde{g}_t))$ $x_{t+1} \gets FW(x_t,\bar{g}_t )$}}; 
\path (9,6.2)                 node [block] (C93) {W$_M$:\\
		\black{$\bar{g}_t \!\gets \!VR(\bar{g}_{t-1} ,\Phi^{\prime}(\tilde{g}_t))$ $x_{t+1} \gets FW(x_t,\bar{g}_t )$}};

\node at (1.8,6) {\ldots};
\node at (7.8,6) {\ldots};
\end{tikzpicture}

%% file: quantization.tex
\section{Gradient Quantization Schemes} \label{sec:encoding schemes}
As mentioned  earlier, the communication cost can be reduced effectively by sending 
quantized gradients. In this section, we introduce a 
quantization 
scheme called  s-\mscheme. Consider the gradient vector $g\in \reals^d$ and let $g_i$ be the $i$-th 
coordinate of the gradient. The s-\mscheme encodes 
$g_i$ into an element from the set 
$\{\pm1,\pm\frac{s-1}{s},\cdots,\pm\frac{1}{s},0\}$ in a random way. To do so, 
we first compute the ratio ${|g_i|}/{\|g\|_\infty}$ and find the 
indicator $l_i \in \{0,1,\cdots,s-1\}$ such that $|g_i|/\|g\|_\infty \in 
[l_i/s, (l_i+1)/s]$. 
Then we define the random variable $b_i$ as
\begin{equation}\label{s_multi_prob_dist}
b_i=
\begin{cases}
l_i/s, & \quad \text{w.p.}\ \ 
1-\frac{|g_i|}{\|g\|_\infty}s+l_i, \\
(l_i+1)/s, & \quad \text{w.p.}\ \  
\frac{|g_i|}{\|g\|_\infty}s-l_i.
\end{cases}
\end{equation}
Finally, instead of transmitting $g_i$, we send $\sign(g_i) \cdot b_i$, 
alongside the norm 
$\|g\|_\infty$. It can be verified that 
$\expect[b_i|g]= {|g_i|}/{\|g\|_\infty}$. So we define the corresponding 
decoding scheme as $\phi^{\prime}(g_i) = \sign(g_i) b_i\|g\|_\infty$ to ensure 
that $\phi^{\prime}(g_i)$ is an unbiased estimator of $g_i$. We note that this 
quantization scheme is similar to the Stochastic Quantization method in 
\citep{alistarh2017qsgd}, except that we use $\ell_\infty$-norm while they adopt the 
$\ell_2$-norm.
In the $s$-\mscheme,  for each coordinate $i$, we need 1 bit to 
transmit 
$\sign(g_i)$. Moreover, since $b_i \in \{0,{1}/{s},\dots, (s-1)/{s}, 
1\}$, we need $z=\log_2(s+1)$
bits to send $b_i$. Finally, we need 32 bits to transmit $\|g\|_\infty$. Hence, 
the total number of communicated bits is  $32+d(z+1)$. Here, by ``bits'' we 
mean the number of 0’s and 1’s transmitted.

One major advantage of the $s$-\mscheme is that by tuning the 
partition parameter $s$ or the corresponding assigned bits $z$, we can smoothly 
control the trade-off between gradient 
quantization and information loss, which helps distributed 
algorithms to attain their best performance. We proceed to characterize the variance of the $s$-\mscheme. 

\begin{lemma}
	\label{lem:var_multi}
	The variance of  $s$-\mscheme $\phi$ for any $g\in \reals^d$ is 
	bounded by 
	\begin{equation}\label{eq:var_multi_claim}
	\var[\phi^{\prime}(g)|g] \leq \frac{d}{s^2}\|g\|_\infty^2.
	\end{equation}
\end{lemma}

If we set $s=1$, we obtain the \sscheme,
which requires communicating the encoded 
scalars $\sign(g_i) b_i \in \{\pm1,0\}$ and the norm $\|g\|_{\infty}$. Since $z 
= 
\log_2(s+1)=1$, the overall communicated bits for each worker are 
$32+2d$ per round. We characterize its variance in \cref{lem:var_single}.


\begin{lemma}
\label{lem:var_single}
	The variance of  \sscheme is given by
	\begin{equation}\label{lem:var_single_claim}
	\var[\phi^{\prime}(g)|g]= \|g\|_1\|g\|_\infty - \|g\|_2^2.
	\end{equation}
\end{lemma}
\begin{remark}
	For the probability distribution of the random variable $b_i$, 
	instead of   $\|g\|_\infty$, we can use  other norms $\|g\|_p$ (where $ p 
	\geq 1$). But it can be verified that the $\ell_\infty$-norm
	leads to the smallest variance for \sscheme. That is also the reason why we 
	do not use 
	$\ell_2$-norm as in \citep{alistarh2017qsgd}.
\end{remark}

%% file: stochastic.tex
\section{Stochastic Optimization}
In this section, we aim to solve the constrained stochastic optimization 
problem defined in \eqref{stochastic_problem} in a distributed fashion. 
In particular, we are interested in projection-free (Frank-Wolfe 
type) methods and execute quantization to reduce the communication cost between 
the workers and the master. Recall that we assume at each round $t$, 
each worker $m \in [M]$ has access to an unbiased estimator of the objective 
function 
gradient $\nabla f(x_t)$, which is denoted by $g_t^m(x_t)$, \emph{i.e.}, 
$\nabla f(x_t) = \expect[g_t^m(x_t) |x_t]$. We further assume that the 
stochastic gradients are independent of each other.

\begin{algorithm}[t]
	\begin{algorithmic}[1]
		\STATE {\bfseries Input:} constraint set $\constraint$, iteration 
		horizon $T$, 
		initial point $x_1 \in \constraint$, $\bar{g}_0 \gets 0$, step sizes 
		$\rho_t, 
		\eta_t$
		\STATE {\bfseries Output:} $x_{T+1}$ 
		or $x_o$, where $x_o$ is chosen from $\{x_1, x_2, \cdots, x_T\}$ 
		uniformly at random
		\FOR{$t=1$ {\bfseries to} $T$}
		\STATE Each worker $m$ gets an independent 
		stochastic 
		gradient $g_t^m(x_t)$
		\STATE Each worker $m$ encodes its local gradient as 
		$\Phi(g_t^m(x_t))$, and pushes  
		$\Phi(g_t^m(x_t))$ to the master
		\STATE Master decodes $\Phi(g_t^m(x_t))$ as 
		$\Phi^{\prime}(g_t^m(x_t))$
		\STATE Master computes the average gradient $\tilde{g}_t \gets 
		\frac{1}{M}\sum_{m=1}^M \Phi^{\prime}(g_t^m(x_t))$
		\STATE Master encodes $\tilde{g}_t$ as $\Phi(\tilde{g}_t)$ and 
		broadcasts 
		it to all the workers
		\STATE Workers decode $\Phi(\tilde{g}_t)$ as 
		$\Phi^{\prime}(\tilde{g}_t)$
		\STATE Workers compute the momentum-based gradient
		$\bar{g}_t \gets (1-\rho_t)\bar{g}_{t-1} + 
		\rho_t\Phi^{\prime}(\tilde{g}_t)$
		\STATE Workers update $x$ based on $x_{t+1} \gets x_t + 
		\eta_t(v_t-x_t)$ where $v_{t} \gets 
		\argmin_{v \in \constraint} \langle v, \bar{g}_t \rangle $ 
		\ENDFOR
	\end{algorithmic}
	\caption{\Alg (S-QFW)\label{alg:dist_fw}}
\end{algorithm}

In our proposed \Alg (S-QFW) method, at iteration $t$, each 
worker $m$ first computes its local stochastic gradient $g_t^m(x_t)$. Then, it 
encodes $g_t^m(x_t)$ as $\Phi(g_t^m(x_t))$ -- which is 
quantized and can be transmitted at a low communication cost -- to the master. 
Once the master receives all the coded stochastic gradients 
$\{\Phi(g_t^m(x_t))\}_{m=1}^M$, it uses a proper decoding scheme to 
evaluate $\{\Phi'(g_t^m(x_t))\}_{m=1}^M$, which are the decoded versions of 
the received signals $\{\Phi(g_t^m(x_t))\}_{m=1}^M$. Indeed, by design, each of 
the decoded signals $\Phi'(g_t^m(x_t))$ is an unbiased estimator of the 
objective function gradient $\nabla f(x_t)$. Then, the master evaluates the 
average of the decoded signals denoted by $\tilde{g}_t$, \emph{i.e.}, 
$\tilde{g}_t =(1/M) \sum_{m=1}^M \Phi'(g_t^m(x_t))$. After using a proper 
quantization scheme, the master broadcasts the coded signal $\Phi(\tilde{g}_t)$ 
to all the workers. The workers decode the received signals and use the 
resulted $\Phi'(\tilde{g}_t)$ vector to improve their gradient 
approximation. 

Note that even in the unquantized setting, if we use the stochastic gradient 
$g_t^m(x_t)$, instead of $\nabla 
f(x_t)$, Frank-Wolfe may still diverge \cite{mokhtari2018stochastic}. As a 
result, we 
need to further reduce the variance. To do so, each worker $m$ uses a momentum 
local vector $\bar{g_t}$ to update the iterates, which is defined by 
\begin{equation}\label{averaging_update}
\bar{g}_t \gets 
(1-\rho_t)\bar{g}_{t-1} + \rho_t\Phi^{\prime}(\tilde{g}_t).
\end{equation}
As the update of $\bar{g}_t $ in \eqref{averaging_update} computes a weighted 
average of the previous stochastic gradient approximation $\bar{g}_{t-1}$ and 
the updated network average stochastic gradient $\Phi^{\prime}(\tilde{g}_t)$, 
it has a lower variance comparing to the vector $\Phi^{\prime}(\tilde{g}_t)$. The key fact 
that allows us to prove convergence  is that the estimation error of 
$\bar{g}_{t}$ approaches zero as time passes (check \cref{lem:new_variance_reduction} in \cref{app:var_reduction_stochastic}). After computing the gradient estimation $\bar{g}_t$ based on \eqref{averaging_update}, workers update their variables by following the 
FW scheme, \emph{i.e.}, $x_{t+1} = x_t + \eta_t(v_t-x_t)$, where 
$v_{t} = \argmin_{v \in \constraint} \langle v, \bar{g}_t \rangle.$
S-QFW is outlined in \cref{alg:dist_fw}. 
Finally, note that we can use different quantization schemes $\Phi$ in  
S-QFW, which leads to different convergence rates 
and communication costs. 

Now we proceed to analyze \sqfw and first focus on convex settings. 
\begin{assump}
	\label{assump_on_K}
	The constraint set $\mathcal{K}$ is convex and 
	compact, with diameter $D = \sup_{x,y \in \mathcal{K}} \|x-y 
	\|$. 
\end{assump}

\begin{assump}
	\label{assump_on_f}
	The 
	function $f$ is convex,  bounded, i.e., $\sup_{x \in \constraint}|f(x)| 
	\leq M_0$, and $L$-smooth  over 
	$\constraint$.
\end{assump}

\begin{assump}
	\label{assump_on_estimate}
	For each worker $m$ and iteration $t$, the stochastic gradient $g_t^m$ is 
	unbiased and has a 
	uniformly bounded variance, \emph{i.e.}, for all $m \in [M]$ and $ t \in 
	[T]$,	
	\begin{equation*}
	\expect[g_t^m(x_t)|x_t]=\nabla 
	f(x_t), \quad \var[g_t^m(x_t)|x_t] \leq \sigma_1^2.
	\end{equation*} 
\end{assump}

\begin{assump}\label{assump_on_quantization}
	For any $x_t \in \constraint$, and vectors $g_t^m(x_t)$ and $\tilde{g}_t$ generated by 
	\Alg, the quantization scheme $\Phi$ satisfies	
	\begin{align*}
	&\expect[\Phi^{\prime}(g_t^m(x_t))|g_t^m(x_t)]=g_t^m(x_t),\qquad\ \ \expect[\Phi^{\prime}(\tilde{g}_t)|\tilde{g}_t]=\tilde{g}_t, \nonumber\\
	&\expect[\|\Phi^{\prime}(g_t^m(x_t)) - g_t^m(x_t)\|^2]\leq \sigma_2^2, 
	\qquad \expect[\|\Phi^{\prime}(\tilde{g}_t) - \tilde{g}_t\|^2]\leq 
	\sigma_3^2.
	\end{align*}
\end{assump}


\begin{theorem}[Convex]
	\label{thm:distributed}
	Under Assumptions
	\ref{assump_on_K} to \ref{assump_on_quantization}, 
	if we set $\eta_t 
	= 2/(t+3), \rho_t = 2/(t+3)^{2/3}$ in \cref{alg:dist_fw}, 
	then after $T$ iterations, the output 
	$x_{T+1} \in \constraint$ satisfies
	\begin{equation*}
	\expect[f(x_{T+1})]-f(x^*) \leq \frac{Q_0}{(T+4)^{1/3}},  
	\end{equation*}
	where $Q_0=\max\{4M_0, 2D(Q^{1/2}+LD)\}, Q = \max 
	\{3\|\nabla f(x_1) \|^2, 
	4(\sigma_1^2+\sigma_2^2)/M+4\sigma_3^2+8L^2D^2 \}$, and $x^*$ is 
	a global minimizer of $f$ on $\constraint$. 
\end{theorem}

\cref{thm:distributed} shows that the 
suboptimality gap of \sqfw converges to zero at a sublinear rate of 
$\mathcal{O}(1/T^{1/3})$. Hence, after running at most 
$\mathcal{O}(\epsilon^{-3})$ iterations, we can find a solution that is 
$\epsilon$ close to the optimum.
%
We also characterize the exact complexity bound for \sqfw when  the \sscheme is 
used for quantization and show that it obtains an $\eps$-accurate solution 
after $\mathcal{O}(\epsilon^{-3})$ rounds for communication. This result is 
presented in \cref{app:thm_sign} due to space limit. Note that as each 
communication round in \sscheme requires  $(M+1)(32+2d)$ bits, the overall 
communication cost to find an $\epsilon$-suboptimal solution is of $\mathcal{O} 
(Md\epsilon^{-3})$. 

With slightly different parameters, \sqfw can be applied to 
 non-convex settings as well. In \textit{unconstrained} non-convex optimization problems, the gradient norm $\|\nabla f\|$ is usually a 
good measure of convergence as $\|\nabla f\| \to 
0$ implies convergence to a stationary point. However, in 
the constrained setting we study the 
Frank-Wolfe Gap \citep{jaggi2013revisiting,lacoste2016convergence} defined as
\begin{equation}
\label{eq:fw_gap}
\mathcal{G}(x) = \max_{v \in \constraint} \langle v-x, -\nabla f(x)\rangle.
\end{equation}
For constrained optimization problem \eqref{stochastic_problem}, if a point 
$x$ satisfies $\mathcal{G}(x)=0$, then it is a first-order stationary point. Also, by  definition, we have $\mathcal{G}(x) 
\geq 0, 
\forall\ x \in \constraint$. 
We analyze the convergence rate of \cref{alg:dist_fw} under the following 
assumption on the objective function $f$.

\begin{assump}
	\label{assump_on_nonconvex_f}
	The function $f$ is bounded, i.e., $\sup_{x \in \constraint}|f(x)| \leq 
	M_0$, and $L$-smooth  over 
	$\constraint$.
\end{assump}

\begin{theorem}[Non-convex]
	\label{thm:distributed_nonconvex}
	Under 
	\cref{assump_on_K,assump_on_nonconvex_f,assump_on_quantization,assump_on_estimate}, and
	given the 
	iteration horizon $T$, if we set $\eta_t =1/(T+3)^{3/4}, \rho_t = 
	2/(t+3)^{1/2}$ in \cref{alg:dist_fw}, then
		\begin{equation*}
	\expect[\mathcal{G}(x_o)] \leq 
	\frac{8M_0+20DQ^{1/2}/3}{(T+3)^{1/4}}+\frac{LD^2}{2(T+3)^{3/4}},
	\end{equation*}
	where $Q = \max \{2\|\nabla f(x_1)\|^2, 
	4(\sigma_1^2+\sigma_2^2)/M+4\sigma_3^2+2L^2D^2 \}$. 
\end{theorem}
\cref{thm:distributed_nonconvex} indicates that in the 
non-convex setting, 
\sqfw finds an $\eps$-first order stationary point after at most 
$\mathcal{O}(\eps^{-4})$ iterations. 
%
%
By using \sscheme, each round  of communication requires $(M+1)(32+2d)$ bits. Therefore, to find an $\epsilon$-first order stationary point, we need $\mathcal{O}(\eps^{-4})$ rounds with the overall communication cost of $\mathcal{O} 
(Md\epsilon^{-4})$. 
%

%% file: finite_sum.tex
\section{Finite-Sum Optimization}

In this section, we analyze the finite-sum problem defined in 
\eqref{finite_sum_problem}. Recall that we assume that there are $N$ functions 
and $M$ workers in total, and each worker $m$ has access to $n=N/M$ functions 
$f_{m,i}$ for $i \in [n]$. The major 
difference with the stochastic setting is that we can use a more aggressive 
variance reduction for communicating quantized gradients.  \citet{nguyen2017sarah,nguyen2017stochastic,nguyen2019optimal} 
developed the 
StochAstic Recursive grAdient algoritHm (SARAH), a stochastic recursive 
gradient update framework. 
Recently, \citet{fang2018spider} proposed  Stochastic 
Path-Integrated Differential Estimator (SPIDER) technique, a variant of SARAH, 
for unconstrained 
optimization in centralized settings. In this paper, we generalize 
SPIDER to the constrained and distributed settings. 

We first consider the  case where no quantization is performed. Let $p \in \mathbb{N}^+$ be a 
period parameter. At the 
beginning of each period, namely, \text{mod}$(t,p)=1$, each worker $m$, 
computes the average of all its local gradients and sends it to the master. 
Then, master calculates the average of the $M$ received signals and broadcasts 
it to all workers. Then, workers update their gradient estimation $\bar{g}_t$ as 
\begin{equation*}
\bar{g}_{t} = \frac{1}{Mn}\sum_{m=1}^M \sum_{i=1}^n \nabla 
f_{m,i}(x_t).
\end{equation*}
Note $\bar{g}_t$ is identical for all the workers. In the rest
of 
that 
period, \emph{i.e.}, \text{mod}$(t,p)\neq 1$, each 
worker $m$ samples a set of local component functions, denoted as 
$\mathcal{S}_{t}^m$, of 
size $S$ uniformly at random, 
computes the average of these gradients and sends it to master. Then, 
master calculates the average of the $M$ signals and broadcasts it to all 
the workers. The workers update their gradient estimation $g_t$ as
\begin{equation}\label{g_avg_update}
\bar{g}_{t} = \bar{g}_{t-1}+\frac{1}{MS}\sum_{m=1}^M\sum_{i 
	\in \mathcal{S}_{t}^m}[\nabla f_{m,i}(x_t)-\nabla 
f_{m,i}(x_{t-1})] .
\end{equation}
So $\bar{g}_t$ is still identical for all the workers.
In order to incorporate quantization, each worker simply pushes the quantized 
version of the average gradients. Then the 
master decodes the quantizations, encodes the average of decoded 
signals in a quantized fashion, and broadcasts the quantization. Finally, 
each worker 
decodes the quantized signal and updates $x_t$ 
locally. 
%
%
%
The full description of our proposed \fAlg (F-QFW) algorithm is outlined 
in \cref{alg:finite_dist_fw}.

\begin{algorithm}[t]
	\begin{algorithmic}[1]
		\STATE {\bfseries Input:} $\constraint$, $T$, No. of workers $M$, initial 
		point $x_1 \in \constraint$, period parameter $p$, sample size $S$
		\STATE {\bfseries Output:} $x_{T+1}$ 
		or $x_o$, 
		where $x_o$ is chosen from $\{x_1, x_2, 
		\cdots, x_T\}$ uniformly at random
		\FOR{$t=1$ {\bfseries to} $T$}
		\IF {\text{mod}$(t,p)=1$} 
		\STATE Each worker $m$ computes its local gradient 
		$g_t^m(x_t)=\frac{1}{n}\sum_{i=1}^n \nabla f_{m,i}(x_t)$
		\STATE Each worker $m$ encodes $g_t^m(x_t)$ as $\Phi_{1,t}(g_t^m(x_t))$ 
		and pushes it to the master
		\STATE Master decodes $\Phi_{1,t}(g_t^m(x_t))$ as 
		$\Phi_{1,t}^{\prime}(g_t^m(x_t))$
		\STATE Master computes the average gradient $\tilde{g}_t \gets 
		\frac{1}{M}\sum_{m=1}^M \Phi_{1,t}^{\prime}(g_t^m(x_t))$
		\STATE Master encodes $\tilde{g}_t$ as $\Phi_{2,t}(\tilde{g}_t)$, 
		and broadcasts it to all workers
		\STATE Workers decode
		$\Phi_{2,t}(\tilde{g}_t)$ as $\Phi_{2,t}^{\prime}(\tilde{g}_t)$ and 
		update $\bar{g}_t \gets \Phi_{2,t}^{\prime}(\tilde{g}_t)$
		\ELSE 
		\STATE Each worker $m$ at time $t$ samples $S$ component functions 
		uniformly at random called $\mathcal{S}_{t}^m$
		\STATE Each worker $m$ computes exact gradients $\nabla f_{m,i}(x_t), 
		\nabla f_{m,i}(x_{t-1})$ for all $i \in \mathcal{S}_{t}^m$ 
		\STATE Each worker $m$ encodes
		$\Phi_{1,t}(\frac{1}{S}{\sum_{i 
				\in \mathcal{S}_{t}^m}[\nabla f_{m,i}(x_t)-\nabla 
			f_{m,i}(x_{t-1})]})$ 
		and pushes to master
		\STATE Master decodes the signals 
		$\Phi_{1,t}(\frac{1}{S}{\sum_{i 
				\in \mathcal{S}_{t}^m}[\nabla f_{m,i}(x_t)-\nabla 
			f_{m,i}(x_{t-1})]})$
		\STATE Master computes $\tilde{g}_t \gets 
		\frac{1}{M}\sum_{m=1}^M \Phi_{1,t}^\prime 
		(\frac{1}{S}{\sum_{i 
				\in \mathcal{S}_{t}^m}[\nabla f_{m,i}(x_t)-\nabla 
			f_{m,i}(x_{t-1})]})$
		\STATE Master encodes $\tilde{g}_t$ as $\Phi_{2,t}(\tilde{g}_t)$, 
		and broadcasts all workers
		\STATE Workers decode
		$\Phi_{2,t}(\tilde{g}_t)$ as $\Phi_{2,t}^{\prime}(\tilde{g}_t)$ and 
		update $\bar{g}_t \gets \Phi_{2,t}^{\prime}(\tilde{g}_t)+\bar{g}_{t-1}$
		\ENDIF
		\STATE Each worker updates $x_{t+1}$ locally by $x_{t+1} \gets 
		x_t 
		+ \eta_t(v_t-x_t)$ where $v_{t} \gets \argmin_{v 
			\in \constraint} \langle v, \bar{g}_t \rangle $
		\ENDFOR
	\end{algorithmic}
	\caption{\fAlg (F-QFW)\label{alg:finite_dist_fw}}
\end{algorithm}

To analyze the convex case, we first make an assumption on the component 
functions.

\begin{assump}
	\label{assump_on_f_i_convex}
	The functions $f_{m,i}$ are convex, $L$-smooth on 
	$\constraint$, and uniformly bounded, i.e.,  $\sup_{x \in 
		\constraint}|f_{m,i}(x)| \leq M_0.$ 
	We also assume that $\sup_{x \in 
		\constraint}\|\nabla f_{m,i}(x)\|_\infty \leq G_\infty, \forall\ m \in 
	[M], i \in [n]$.
\end{assump}
	\begin{theorem}[Convex]
		\label{thm:finite_convex}
		Consider F-QFW outlined in Algorithm~\ref{alg:finite_dist_fw}. Recall 
		that $n$ indicates the number of local functions at each node, and $S$ 
		indicates the size of mini-batch used in~\eqref{g_avg_update}. 
		Under \cref{assump_on_K,assump_on_f_i_convex},
		if we set $ p=\sqrt{n}$, $S=\sqrt{n}$, and 
		$\eta_t={2}/({p\lceil \frac{t}{p} \rceil )}$, and use the 
		$s_{1,t}=({pd^{1/2}S^{1/2}M^{-1/2}}\lceil 
		\frac{t}{p} \rceil)$-\mscheme, and $s_{2,t}=(pd^{1/2}S^{1/2}\lceil 
		\frac{t}{p} \rceil)$-\mscheme as 
		$\Phi_{1,t}$ and $\Phi_{2,t}$ in \cref{alg:finite_dist_fw}, 
		then the output $x_{T+1} \in \constraint$  satisfies
		\begin{equation*}
		\expect[f(x_{T+1})]-f(x^*) \leq {Q_0}/{T},  
		\end{equation*}
		where  $Q_0=\max\{6pM_0, 3Q\}$, 
		$Q=4D(\sqrt{L^2D^2+2G_\infty^2}+LD)$, and $x^*$ 
		is a minimizer of $f$ on $\constraint$. 
		
	\end{theorem}
%


\cref{thm:finite_convex} indicates that in convex setting, if we 
use the recommended quantization schemes, then the output of 
\fAlg is 
$\epsilon$-suboptimal with at most ${Q_0}/{\epsilon}$ rounds.  
As $p=\sqrt{n}$, the Linear-optimization Oracle  (LO) complexity is 
$\mathcal{O}(\sqrt{n}/\epsilon)$. Also, the total Incremental First-order 
Oracle  (IFO) complexity is 
$[Mn+2(p-1)MS] \times ({T}/{p})=\mathcal{O}(n/\epsilon)$. By considering the 
quantization schemes with $s_{1,t}$ and $s_{2,t}$ quantization levels, the 
average communication bits per round are at most 
$d(M\lceil\log_2[({\sqrt{n}dT^2}/{M})^{1/2}+1]\rceil+
\lceil\log_2[(\sqrt{n}dT^2)^{1/2}+1]\rceil)+(M+1)(d+32)$.



\cref{alg:finite_dist_fw} can also be applied to the non-convex setting with a 
slight change in
parameters. 
We first make a standard assumption on 
the component functions. 

\begin{assump}
	\label{assump_on_f_i}
	The component functions $f_{m,i}$ are $L$-smooth on $\constraint$ and 
	uniformly bounded, \emph{i.e.},  
	$\sup_{x \in 
		\constraint}|f_{m,i}(x)| \leq M_0$. We 
	also assume that $\sup_{x \in 
	\constraint}\|\nabla f_{m,i}(x)\|_\infty \leq G_\infty, \forall\ m \in 
[M], i \in [n]$.
\end{assump}


	\begin{theorem}[Non-convex]\label{thm:quantized_finite_distributed_nonconvex}
		Under \cref{assump_on_K,assump_on_f_i},		
		if we set $ p=\sqrt{n}$, $S=\sqrt{n}$, and $\eta_t=T^{-1/2}$, 
		and use the  
		$s_{1,t}=({4\sqrt{n}dT}/{M})^{1/2}$-\mscheme, and  
		$s_{2,t}=((4\sqrt{n}dT)^{1/2})$-\mscheme as 
		$\Phi_{1,t}$ and $\Phi_{2,t}$ in \cref{alg:finite_dist_fw}, then the output $x_o \in \constraint$  satisfies
		\begin{equation*}
		\expect[\mathcal{G}(x_o)] \leq 
		\frac{2M_0+D\sqrt{L^2D^2+2G_\infty^2}+{LD^2}}{\sqrt{T}}.
		\end{equation*}
	\end{theorem}

\cref{thm:quantized_finite_distributed_nonconvex} shows that for 
non-convex minimization, if we adopt the recommended quantization schemes, then 
\cref{alg:finite_dist_fw} finds an $\epsilon$-first order stationary point 
with at most 
$\mathcal{O}(1/\epsilon^2)$ rounds. Also, the total IFO complexity is 
$[Mn+2(p-1)MS]\cdot \frac{T}{p}= 
\mathcal{O}(\sqrt{n}/\epsilon^2)$, and the average communication bits per round 
are $d(M\lceil \log_2[({4\sqrt{n}dT}/{M})^{1/2}+1] 
\rceil+\lceil \log_2[(4\sqrt{n}dT)^{1/2}+1] \rceil)+(M+1)(d+32)$.




%% file: experiment.tex
\section{Experiments}
\label{sec:experiments}

We evaluate the performance of algorithms by visualizing their optimality gap 
$ 
f(x_t)-f(x^*) $ (for convex settings), their loss $f(x_t)$ (for non-convex settings) as well as their testing accuracy vs.\ the number of transmitted bits. The experiments were performed on 20 Intel Xeon E5-2660 cores 
and thus the number of workers is 20. For each curve in the figures below, we 
ran at least 50 repeated experiments, and the height of shaded regions 
represents two 
standard deviations.

In our first setup, we consider a multinomial logistic regression problem. Consider the dataset $ \{ (x_i,y_i)  \}_{i=1}^N \subseteq \mathbb{R}^d\times 
\{1,\dots, C\}$ with $N$ samples that have $C$ different labels. We aim to find a model $w$ to classify these sample points under the condition that the solution has a small $\ell_1$-norm. Therefore, we aim to solve the following convex problem
\begin{equation}\label{mul_log}
\min_{w} f(w) := -\sum_{i=1}^N \sum_{c=1}^C 1\{y_i=c\} \log \frac{\exp(w_c^\top 
x_i)}{\sum_{j=1}^C \exp(w_j^\top x_i)}, \quad \text{s.t. } {\|w\|_1 \le 1}.
\end{equation}
In our experiments, we use the MNIST and CIFAR-10 datasets. For the 
MNIST dataset, we assume that each worker stores $ 
3000$ images, and, therefore, the overall number of samples in the training set 
is $N=60000$. The result on 
CIFAR-10 is similar and deferred to \cref{app:additonal_experiments}.

In our second setup, our goal is to minimize the loss of a three-layer neural network under some conditions on the norm of the solution. 
Before stating the problem precisely, let us define the log-loss function as $
h(y,p)\triangleq -\sum_{c=1}^C 1\{ y=c \}\log p_c $ for $ y\in 
\{1,\dots,C\} $ and a $ C $-dimensional probability vector $ p:= (p_1, \cdots, p_C) $. We aim to solve the following non-convex problem 
\begin{equation}\label{nn}
\min_{W_1, W_2} f(W_1,W_2,b_1,b_2) := \sum_{i=1}^N h(y_i, 
\phi(W_2\sigma(W_1x_i\!+\!b)\!+\!b_2) ), \ \ \ \text{s.t.} \  \|W_i\|_1\le a_1, 
\|b_i\|_1\le 
a_2,  
\end{equation}
where $ \sigma(x)\triangleq (1+e^{-x})^{-1} $ is the sigmoid function and $\phi$ is the softmax function. The imposed $\ell_1$ constraint on the weights leads to a sparse network. We further remark that Frank-Wolfe methods are suitable for training a neural network subject to an $\ell_1$ constraint as they are equivalent to a dropout 
regularization \citep{ravi2018constrained}. 
We use the MNIST 
and CIFAR-10 datasets. 
For 
the MNIST dataset, we assume that each  worker stores $3000 $ images. The size 
of matrices $ W_1 $ and $ W_2 $ are $ 784\times 10 $ and $ 10\times 10 $, 
respectively, and the constraints parameters are $a_1=a_2=10$. 
We obtain a similar result on CIFAR-10 and discuss it in 
\cref{app:additonal_experiments}.

\begin{figure*}[t]
	\centering
			\begin{subfigure}[b]{0.45\linewidth}
				\includegraphics[width=\linewidth]{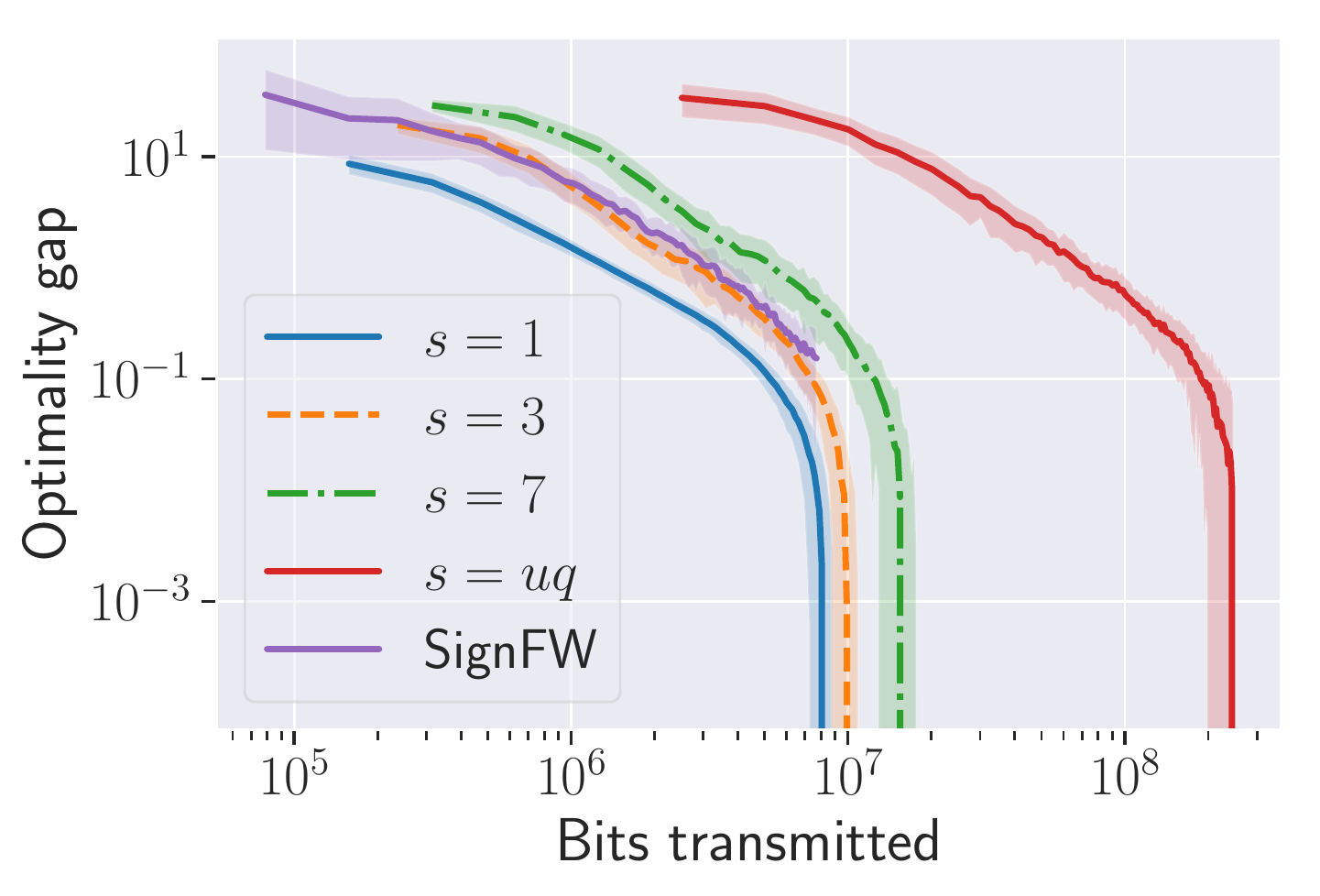} 
				\caption{Optimality gap vs.\ bits transmitted}
				\label{fig:sc-opt}
			\end{subfigure}\hfill
		\begin{subfigure}[b]{0.45\linewidth}
			\includegraphics[width=\linewidth]{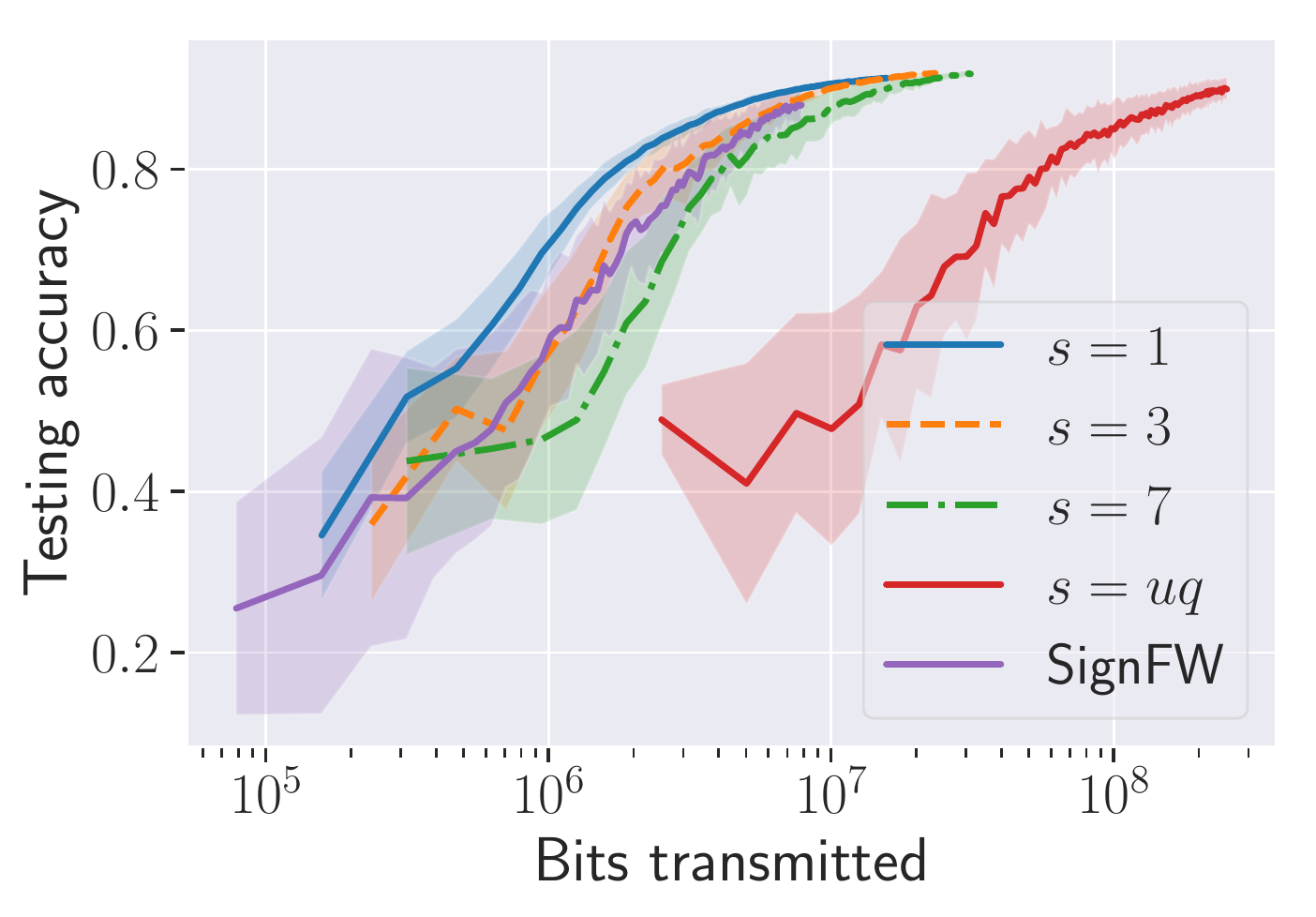} 
			\caption{Testing accuracy vs.\ bits transmitted}
			\label{fig:sc-accu}
		\end{subfigure} 
	\caption{Comparison in terms of optimality gap (left) and test accuracy 
	(right) versus number of transmitted bits for a multinomial logistic 
	regression problem. The best performance belongs to \qfwshort with \sscheme 
	($ s=1 $), and FW without quantization has the worst performance.
		\label{fig:stoch_1}
	}
\end{figure*}

\begin{figure*}[t]
	\centering
			\begin{subfigure}[b]{0.45\linewidth}
				\includegraphics[width=\linewidth]{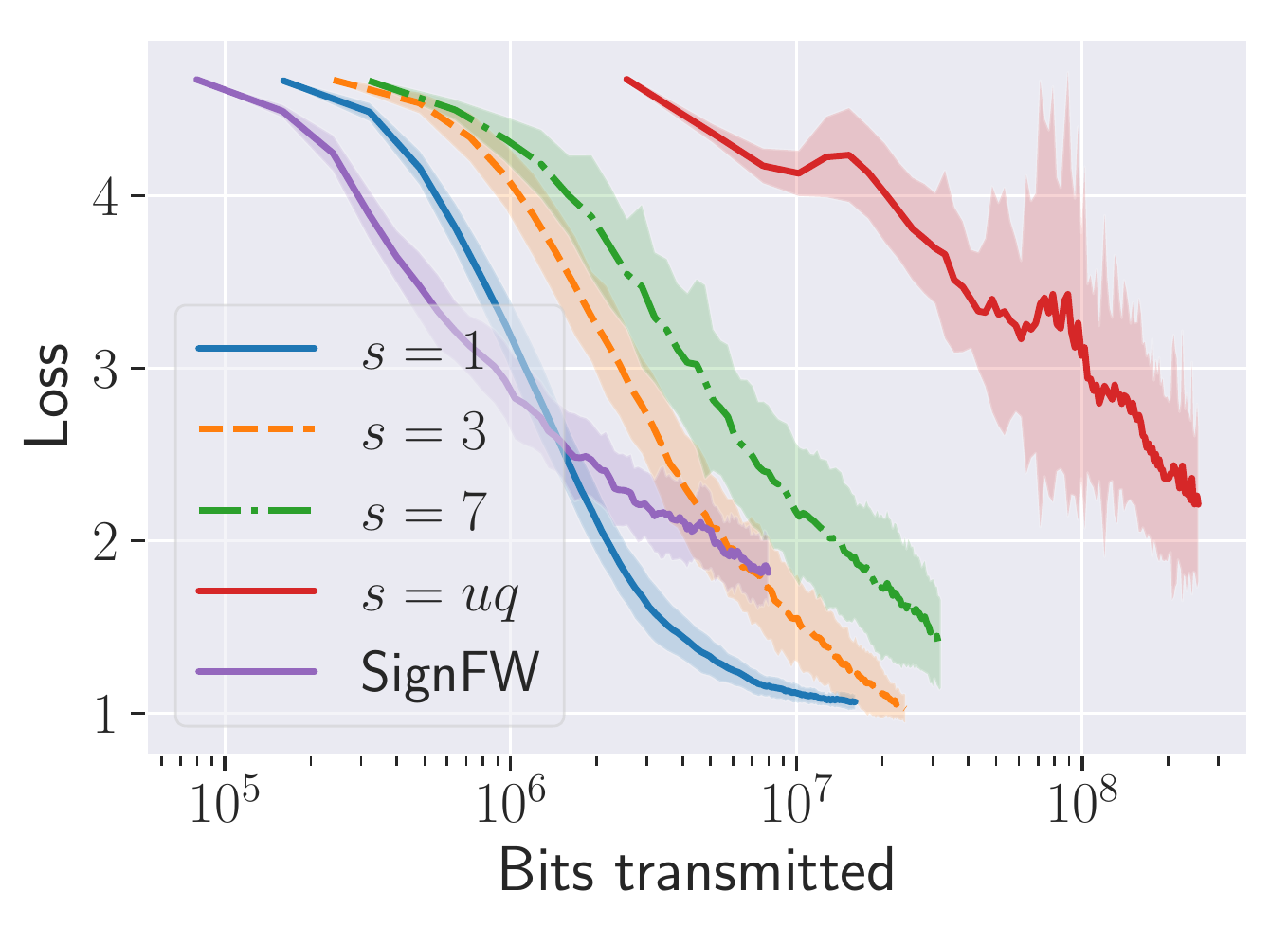} 
				\caption{Loss vs.\ bits transmitted}
				\label{fig:sn-loss}
			\end{subfigure}\hfill
			\begin{subfigure}[b]{0.45\linewidth}
				\includegraphics[width=\linewidth]{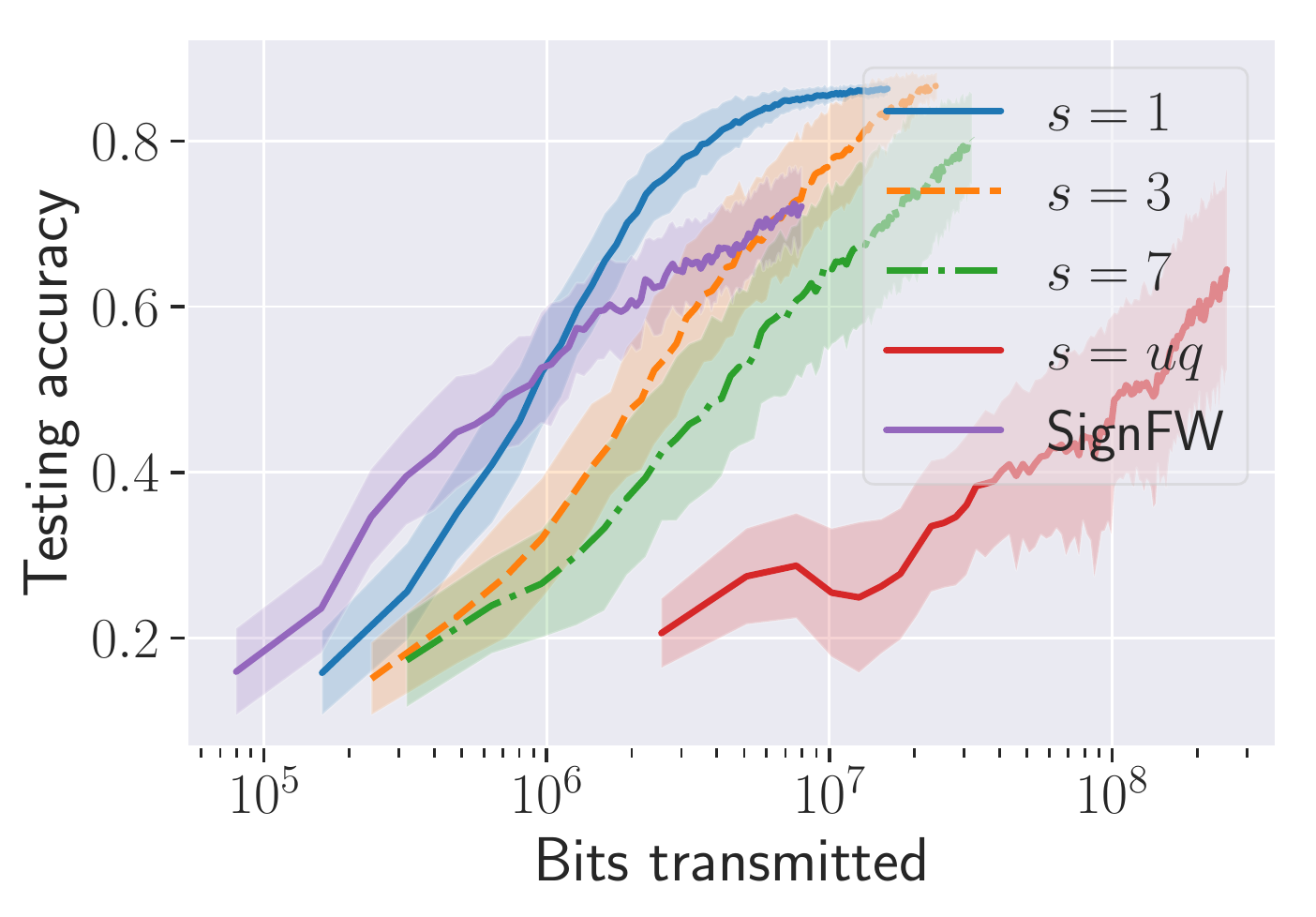} 
				\caption{Testing accuracy vs.\ bits transmitted}
				\label{fig:sn-accu}
			\end{subfigure}
	\caption{Comparison of algorithms in terms of loss function (left) and test 
	accuracy (right) versus number of transmitted bits for a three-layer neural 
	network. FW without quantization ($s=uq$) significantly underperforms the 
	quantized FW methods.  
		\label{fig:stoch_2}
	}
\end{figure*}

In our third setup, we study a multi-task least square regression 
problem~\citep{zheng2018distributed}. Its 
setting and result are discussed in \cref{app:additonal_experiments}.

For all of the considered settings,  we vary the quantization level and use the 
$s$-partition encoding scheme ($s=uq$ indicates FW without quantization). We 
also propose \signfw, an effective heuristic based on QFW, where the norm of 
the gradient is discarded and only the sign of each coordinate is transmitted. 
Even though this method may not enjoy the strong theoretical guarantees of QFW 
(and may even diverge) we observed in our experiments that it performs on par 
with QFW in practice. Let us emphasize that the proposed  \signfw algorithm is 
similar to \qfwshort with \sscheme except that $ 
\|g\|_\infty $ is not transmitted  and only $ \sign(g_i)b_i $ is transmitted (see Section~\ref{sec:encoding schemes}).


In Figure \ref{fig:stoch_1}, we observe the performance of \signfw, FW without quantization, and different variants of  \qfwshort for solving the multinomial logistic regression problem in \eqref{mul_log}. We observe that 
 \qfwshort with \sscheme ($ s=1 $) has the best performance and all quantized variants of FW outperform the FW method without quantization both in terms of training error and test accuracy. Specifically, \qfwshort with \sscheme ($ s=1 $) requires $ 8\times 10^6 $ transmitted bits to hit the lowest optimality gap in 
\cref{fig:sc-opt}, while \qfwshort with $ s=3 $ and $ s=7 $ require $ 10^7 $ 
and $ 1.5\times 10^7 $ bits, respectively, for achieving the same error. Furthermore,  FW without quantization requires more than $ 2\times 10^8 
$ bits to reach the same error, \emph{i.e.}, quantization reduces communication 
load by 
at least an order of magnitude.

Figure \ref{fig:stoch_2} demonstrates the performance of \signfw, FW without quantization, and different variants of  \qfwshort for solving the three-layer neural network in \eqref{nn}. The relative behavior of the considered methods in Figure \ref{fig:stoch_2} is similar to the one in Figure \ref{fig:stoch_1}. \qfwshort 
with \sscheme obtains a loss less than $ 2 $ after transmitting $ 2\times 10^6  
$ bits, while to attain the same loss level, it takes $ 5\times 10^6 $ bits if 
one uses \signfw or \qfwshort with $ s=3 $. The number of required bits becomes 
approximately $ 1.5\times 10^7$ for $ s=7 $. Also, if no quantization is 
applied, then the number of required bits is at least the  $ 3\times 10^8 $ 
(\emph{i.e.}, quantization reduces communication load by at least two orders of 
magnitude). To achieve a testing accuracy greater than $ 0.8 $, \qfwshort with 
$ s=1 $ requires $ 3\times 10^6 $ bits transmission, while the second most 
communication-efficient method \qfwshort with $ s=3 $ needs $ 10^7 $ bits. 

%% file: additional_experiments.tex
\begin{figure*}[t]
	\centering
	\begin{subfigure}[b]{0.45\linewidth}
	\includegraphics[width=\linewidth]{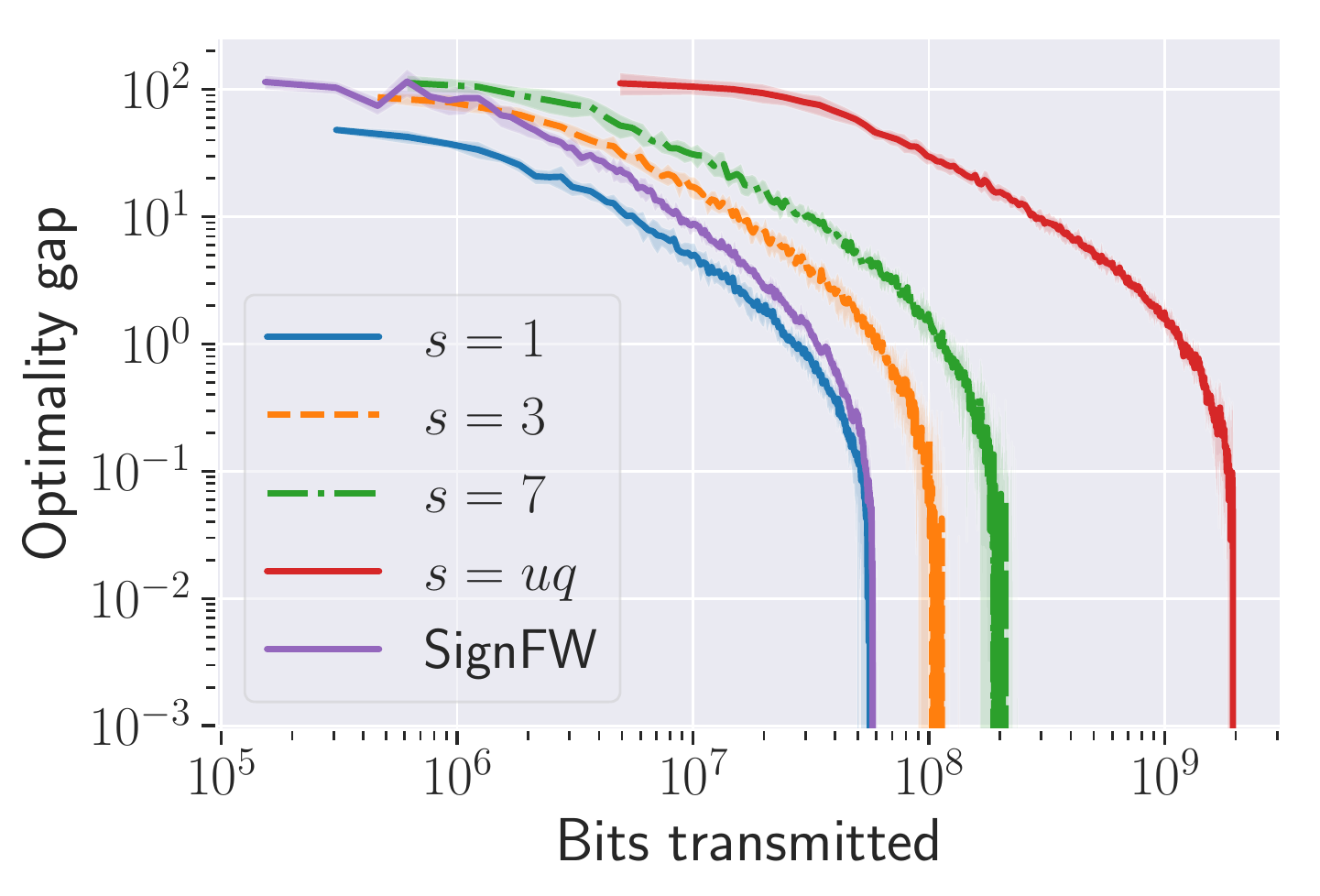} 
	\caption{Multinomial logistic regression}
	\label{fig:cf-opt}
\end{subfigure}\hfill
	\begin{subfigure}[b]{0.45\linewidth}
		\includegraphics[width=\linewidth]{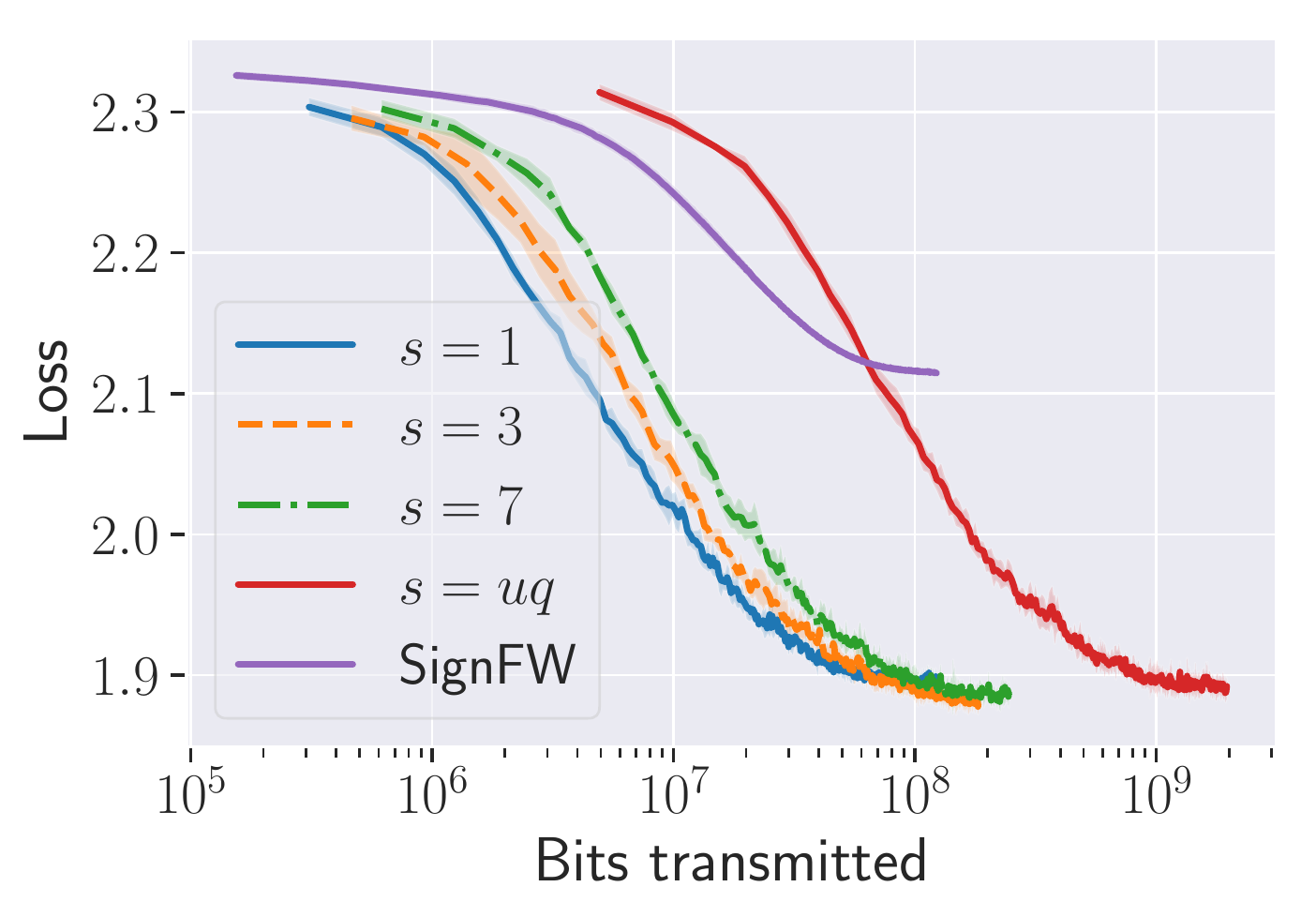} 
		\caption{Three-layer neural network}
		\label{fig:cn-opt}
	\end{subfigure}
	\caption{Comparison in terms of the optimality gap versus the number of 
	transmitted bits for the multinomial logistic 
	regression problem (left) and training a three-layer neural network (right) 
	on the CIFAR-10 dataset. The best performance belongs to \qfwshort with 
	\sscheme 
	($ s=1 $), and FW without quantization ($ s=uq $) has the worst performance.
		\label{fig:cifar-10}
	}
\end{figure*}
\begin{figure}[htb]
	\centering
	\includegraphics[width=0.45\linewidth]{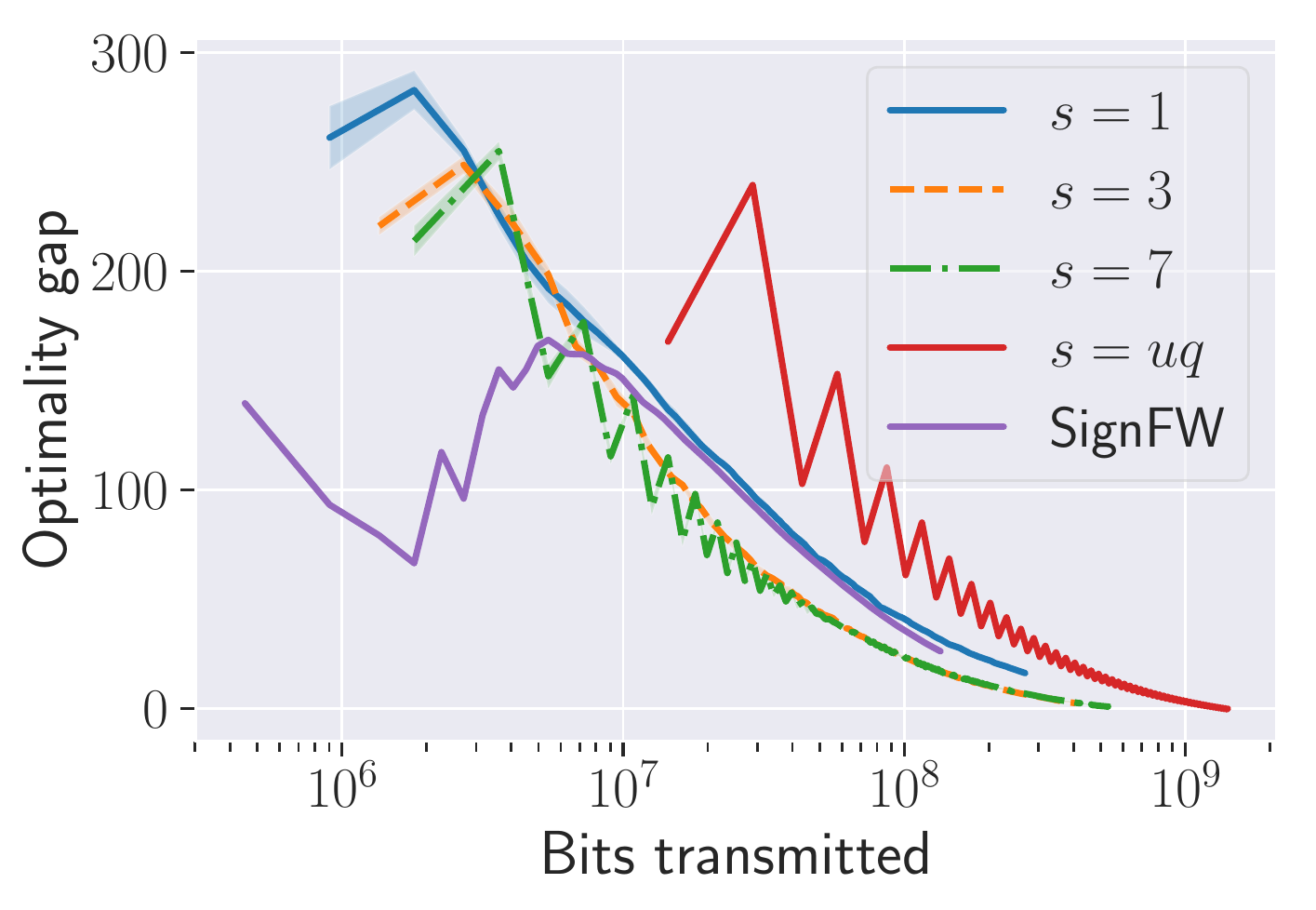} 
	\caption{Optimality gap vs.\ bits transmitted for the task of 
		multi-task least square regression. The Frank-Wolfe algorithm without 
		quantization is denoted by $ s=uq $.}
	\label{fig:mlsr}
\end{figure}

We conduct additional experiments on the CIFAR-10 dataset. The total number of 
images in the training set of CIFAR-10 is $ N = 50000 $. We assume that each 
worker stores 2500 images. We consider the loss function of multinomial 
logistic regression \eqref{mul_log} and the log loss of a three-layer neural 
network \eqref{nn}. In both \eqref{mul_log} and \eqref{nn}, the number of 
classes is $ C=10 $. In \eqref{nn}, the size of matrices $ W_1 $ and $ W_2 $ 
are 
$ 3072 \times 10 $ and $ 10\times 10 $, respectively. The constraints 
parameters are $ a_1=a_2=10 $. As are presented in \cref{fig:cifar-10}, the 
results are similar to those of the experiments on the MNIST dataset. For both 
objective functions, \qfwshort with \sscheme ($ s=1 $) achieves the best 
performance and the Frank-Wolfe without quantization has the worst performance. 

Our third setup studies a multi-task least square regression 
problem~\citep{zheng2018distributed}. We aim to minimize the following 
objective \begin{equation}\label{eq:mlsr}
	\min_W f(W) = \frac{1}{2}\|XW-Y\|_F = \frac{1}{2}\sum_{i=1}^N \sum_{j=1}^m 
	(x_i^\top w_j-y_{ij})^2\quad \text{s.t. } {\|W\|_1 \le 1}\,,
\end{equation}
where $ X\in \mathbb{R}^{N\times d} $ is termed the feature matrix, $ Y\in 
\mathbb{R}^{N\times m} $ is the response matrix, $ W\in \mathbb{R}^{d\times m}  
$ is the weight matrix that we aim to optimize, and $ \|\cdot\|_F $ denotes 
the 
Frobenius norm. In this setup, we use the synthetic data. We set $ N=50000 $ 
and $ m=d=300 $. Every component of $ X $ is sampled from the standard normal 
distribution and the true value of $ W^* $ is $ \frac{W'}{\|W'\|_1} $, where 
every entry of $ W' $ is sampled from the standard normal distribution. Then we 
form $ Y = XW^* $. We use $ 20 $ workers, each storing $ 2500 $ rows of the  
matrices $ X $ and $ Y $. We present the optimality gap versus the number of 
transmitted bits in \cref{fig:mlsr}. It can be observed that the Frank-Wolfe 
algorithm without quantization is the least communication-efficient. Although 
\signfw achieves the smallest optimality gap at the initial stage, \qfwshort 
with $ s=3 $ and $ s=7 $ outperform other algorithms eventually.